%% file: article.tex
\documentclass{article}

\usepackage{graphicx} 
\usepackage{subfigure} 

\usepackage{csquotes}

\usepackage{natbib}

\usepackage{algorithm}
\usepackage{algorithmic}

\usepackage{hyperref}

\usepackage[accepted]{icml2012}

\icmltitlerunning{Approximate Modified Policy Iteration}

\usepackage{mathtools}

\usepackage{amsmath}
\usepackage{amssymb}
\usepackage{amsthm}
\usepackage{dsfont}

\newtheorem{theorem}{Theorem}
\newtheorem{proposition}{Proposition}
\newtheorem{lemma}{Lemma}

\newtheorem{definition}{Definition}
\theoremstyle{definition}
\newtheorem{remark}{Remark}

\newcommand{\fit}[1]{{\operatorname*{Fit}}_{\cal F}\left({#1}\right)}

\newcommand{\appdx}[1]{\citep[Appendix~{#1}]{scherrer2012}}

\newlength{\minipagewidth}
\newcommand{\bookbox}[1]{\small
\par\medskip\noindent
\framebox[\columnwidth]{
\begin{minipage}{\minipagewidth} {#1} \end{minipage} } \par\medskip }

\def\={\stackrel{\Delta}{=}}
\def\E{\mathbb{E}}
\def\R{\mathbb{R}}
\def\1{\mathds{1}}
\def\G{{\cal G}}
\def\F{{\cal F}}

\def\cf{\emph{c.f.} }

\input{symbolsdef}
\allowdisplaybreaks

\begin{document} 

\twocolumn[
\icmltitle{Approximate Modified Policy Iteration}

\icmlauthor{Bruno Scherrer}{Bruno.Scherrer@inria.fr}
\icmladdress{INRIA Nancy - Grand Est, Maia Team, FRANCE}
\vspace{-.1cm}
\icmlauthor{Victor Gabillon}{Victor.Gabillon@inria.fr}
\icmlauthor{Mohammad Ghavamzadeh}{Mohammad.Ghavamzadeh@inria.fr}
\icmladdress{INRIA Lille - Nord Europe, Sequel Team, FRANCE}
\vspace{-.1cm}
\icmlauthor{Matthieu Geist}{Matthieu.Geist@supelec.fr}
\icmladdress{Supélec, IMS Research Group, Metz, FRANCE}


\icmlkeywords{Approximate Dynamic Programming, Reinforcment Learning, Performance Bounds}

\vskip 0.3in
]

\begin{abstract} 
Modified policy iteration (MPI) is a dynamic programming (DP) algorithm that contains the two celebrated policy and value iteration methods. Despite its generality, MPI has not been thoroughly studied, especially its approximation form which is used when the state and/or action spaces are large or infinite. In this paper, we propose three implementations of approximate MPI (AMPI) that are extensions of well-known approximate DP algorithms: fitted-value iteration, fitted-Q iteration, and classification-based policy iteration. We provide error propagation analyses that unify those for approximate policy and value iteration. On the last classification-based implementation, we develop a finite-sample analysis that shows that MPI's main parameter allows to control the balance between the estimation error of the classifier and the overall value function approximation.
\end{abstract} 


\section{Introduction}

Modified Policy Iteration (MPI) \citep{Puterman:1978} is an iterative algorithm to compute the optimal policy and value function of a Markov Decision Process (MDP). Starting from an arbitrary value function $v_0$, it generates a sequence of value-policy pairs 
\begin{align}
\pi_{k+1} & =  \greedy{v_k} & & \mbox{(greedy step)}\label{eq:greedystep}\\
v_{k+1} & =  (T_{\pi_{k+1}})^m v_k & &\mbox{(evaluation step)\label{eq:evaluationstep}}
\end{align}
where $\greedy{v_k}$ is a {\em greedy} policy w.r.t.~$v_k$, $T_{\pi_k}$ is the Bellman operator associated to the policy $\pi_k$, and $m \geq 1$ is a parameter. MPI generalizes the well-known dynamic programming algorithms Value Iteration (VI) and Policy Iteration (PI) for values $m=1$ and $m=\infty$, respectively. MPI has less computation per iteration than PI (in a way similar to VI), while enjoys the faster convergence of the PI algorithm~\citep{Puterman:1978}.
In problems with large state and/or action spaces, approximate versions of VI (AVI) and PI (API) have been the focus of a rich literature (see e.g.~\citealt{Bertsekas:1996,Szepesvari:2010}). 
 The aim of this paper is to show that, similarly to its exact form, approximate MPI (AMPI) may represent an interesting alternative to AVI and API algorithms.

In this paper, we propose three implementations of AMPI (Sec.~\ref{algorithms}) that generalize the AVI implementations of~\citet{Ernst:2005,Antos:2007,Munos:2008} and the classification-based API algorithm of~\citet{Lagoudakis:2003b,Fern:2006,Lazaric:2010,Gabillon11CP}. We then provide an error propagation analysis of AMPI (Sec.~\ref{analysis}), which shows how the $L_p$-norm of its performance loss can be controlled by the error at each iteration of the algorithm. We show that the error propagation analysis of AMPI is more involved than that of AVI and API. This is due to the fact that neither the contraction nor monotonicity arguments, that the error propagation analysis of these two algorithms rely on, hold for AMPI. The analysis of this section unifies those for AVI and API and is applied to the AMPI implementations presented in Sec.~\ref{algorithms}. 
We detail the analysis of the classification-based implementation of MPI (CBMPI) of Sec.~\ref{algorithms} by providing its finite sample analysis in Sec.~\ref{cbmpi}.
Our analysis indicates that the parameter $m$ allows us to balance the estimation error of the classifier with the overall quality of the value approximation.  We report some preliminary results of applying CBMPI to standard benchmark problems and comparing it with some existing algorithms in \appdx{\ref{s:experiments}}.

\section{Background}
\label{background}

We consider a discounted MDP $\langle \S, \A, P, r, \gamma\rangle$, where $\S$ is a state space, $\A$ is a finite action space, $P(ds'|s,a)$, for all $(s,a)$, is a probability kernel on $\S$, the reward function $r:\S \times\A \rightarrow \R$ is bounded by $\Rmax$, and $\gamma \in (0,1)$ is a discount factor.  A deterministic policy is defined as a mapping $\pi:\S \rightarrow\A$. For a policy $\pi$, we may write $r_\pi(s)=r\big(s,\pi(s)\big)$ and $P_\pi(ds'|s)=P\big(ds'|s,\pi(s)\big)$. The value of policy $\pi$ in a state $s$ is defined as the expected discounted sum of rewards received starting from state $s$ and following the policy $\pi$, i.e.,$v_\pi(s)=\E \big[ \left. \sum_{t=0}^{\infty}\gamma^t r_\pi(s_t)\right| s_0=s, s_{t+1}\sim P_\pi(\cdot|s_t)\big].$ Similarly, the action-value function of a policy $\pi$ at a state-action pair $(s,a)$, $Q_\pi(s,a)$, is the expected discounted sum of rewards received starting from state $s$, taking action $a$, and then following the policy. Since the rewards are bounded by $\Rmax$, the values and action-values should be bounded by $\Vmax=\Qmax=\Rmax/(1-\gamma)$.
%
%
The Bellman operator $T_\pi$ of policy $\pi$ takes a function $f$ on $\S$ as input and returns the function $T_\pi f$ defined as $\forall s, ~[T_{\pi}f] (s)=\E\big[ r_\pi(s)+\gamma f(s') ~|~ s' \sim P_\pi(.|s)\big]$, or in compact form, $T_\pi f=r_\pi+\gamma P_\pi f$. It is known that $v_\pi$ is the unique fixed-point of $T_\pi$. Given a function $f$ on $\S$, we say that a policy $\pi$ is greedy w.r.t.~$f$, and write it as $\pi = \greedy f$, if $\forall s,\;(T_\pi f)(s)=\max_a (T_a f)(s)$, or equivalently $T_\pi f=\max_{\pi'} (T_{\pi'} f)$. We denote by $v_*$ the optimal value function. It is also known that $v_*$ is the unique fixed-point of the Bellman optimality operator $T: v \rightarrow \max_\pi T_\pi v = T_{\greedy(v)}v$, and that a policy $\pi_*$ that is greedy w.r.t.~$v_*$ is optimal and its value satisfies $v_{\pi_*}=v_*$.

\section{Approximate MPI Algorithms}
\label{algorithms}

In this section, we describe three approximate MPI (AMPI) algorithms. These algorithms rely on a function space $\F$ to approximate value functions, and in the third algorithm, also on a policy space $\Pi$ to represent greedy policies. In what follows, we describe the iteration $k$ of these iterative algorithms. 


\subsection{AMPI-V}
\label{mpiv}

For the first and simplest AMPI algorithm presented in the paper, we assume that the values $v_k$ are represented in a function space $\F \subseteq \R^{|\S|}$. In any state $s$, the action $\pi_{k+1}(s)$ that is greedy w.r.t.~$v_k$ can be estimated as follows:
\begin{equation}
\label{eq:mpiv_greedy}
\pi_{k+1}(s) = \arg\max_{a\in\A} \frac{1}{M} \Big( \sum_{j=1}^M r^{(j)}_a + \gamma v_k(s^{(j)}_a) \Big),
\end{equation}
where $\forall a\in\A$ and $1\le j\le M$, $r^{(j)}_a$ and $s^{(j)}_a$ are samples of rewards and next states when action $a$ is taken in state $s$. Thus, approximating the greedy action in a state $s$ requires $M|\A|$ samples. The algorithm works as follows. It first samples $N$ states from a distribution $\mu$, i.e.,~$\{s^{(i)}\}_{i=1}^N\sim\mu$. From each sampled state $s^{(i)}$, it generates a rollout of size $m$, i.e.,~$\big( s^{(i)},a^{(i)}_0,r^{(i)}_0,s^{(i)}_1,\dots,a^{(i)}_{m-1},r^{(i)}_{m-1}, s^{(i)}_m \big)$, where $a^{(i)}_t$ is the action suggested by $\pi_{k+1}$ in state $s_t^{(i)}$, computed using Eq.~\ref{eq:mpiv_greedy}, and $r^{(i)}_t$ and $s^{(i)}_{t+1}$ are the reward and next state induced by this choice of action. For each $s^{(i)}$, we then compute a rollout estimate
%
$
\hv_{k+1}(s^{(i)})=  \sum_{t=0}^{m-1} \gamma^t r^{(i)}_t+\gamma^m v_k(s^{(i)}_m),
$
%
which is an unbiased estimate of $\left[\left(T_{\pi_{k+1}}\right)^m v_k\right](s^{(i)})$. Finally, $v_{k+1}$ is computed as the best fit in $\cal F$ to these estimates, i.e.,
\begin{equation*}
\label{eq:mpiv_eval}
v_{k+1} = \fit{\left\{\big(s^{(i)}, \hv_{k+1}(s^{(i)})\big) \right\}_{i=1}^N}.
\end{equation*}
Each iteration of AMPI-V requires $N$ rollouts of size $m$, and in each rollout any of the $|\A|$ actions needs $M$ samples to compute Eq.~\ref{eq:mpiv_greedy}. This gives a total of $Nm(M|\A|+1)$ transition samples. Note that the fitted value iteration algorithm~\citep{Munos:2008} is a special case of AMPI-V when $m=1$.


\subsection{AMPI-Q}


In AMPI-Q, we replace the value function $v:\S \rightarrow \R$ with an action-value function $\q:\S \times \A \rightarrow \R$. The Bellman operator for a policy $\pi$ at a state-action pair $(s,a)$ can then be written as 
\begin{equation*}
[T_\pi \q](s,a)=\E \big[ r_{\pi}(s,a)+\gamma \q(s',\pi(s')) | s' \sim P(\cdot|s,a)\big],
\end{equation*}
and the greedy operator is defined as
\begin{equation*}
\pi=\greedy{\q} ~~\Leftrightarrow~~ \forall s\quad\pi(s)=\arg\max_{a\in\A} \q(s,a).
\end{equation*}
In AMPI-Q, action-value functions $\q_k$ are represented in a function space $\F \subseteq \R^{|\S\times \A|}$, and the greedy action w.r.t.~$\q_k$ at a state $s$, i.e.,~$\pi_{k+1}(s)$, is computed as
\begin{equation}
\label{eq:mpiq_greedy}
\pi_{k+1}(s) \in \arg\max_{a\in\A}\;\q_k(s,a).
\end{equation}
The {\em evaluation step} is similar to that of AMPI-V, with the difference that now we work with state-action pairs. We sample $N$ state-action pairs from a distribution $\mu$ on $\mathcal{S} \times \action$ and build a rollout set $\D_k=\{(s^{(i)},a^{(i)})\}_{i=1}^N,\;(s^{(i)},a^{(i)})\sim\mu$. For each $(s^{(i)},a^{(i)})\in\D_k$, we generate a rollout of size $m$, i.e.,~$\big( s^{(i)},a^{(i)},r^{(i)}_0,s^{(i)}_1,a^{(i)}_1,\cdots,s^{(i)}_m, a^{(i)}_m \big)$, where the first action is $a^{(i)}$, $a^{(i)}_t$ for $t\geq 1$ is the action suggested by $\pi_{k+1}$ in state $s_t^{(i)}$ computed using Eq.~\ref{eq:mpiq_greedy}, and $r^{(i)}_t$ and $s^{(i)}_{t+1}$ are the reward and next state induced by this choice of action. For each $(s^{(i)},a^{(i)})\in\D_k$, we then compute the rollout estimate
\begin{equation*}
\widehat{\q}_{k+1}(s^{(i)},a^{(i)}) = \sum_{t=0}^{m-1} \gamma^t r^{(i)}_t+\gamma^m \q_k(s^{(i)}_m,a_m^{(i)}), 
\end{equation*}
which is an unbiased estimate of $\big[(T_{\pi_{k+1}})^m  \q_k\big](s^{(i)},a^{(i)})$. Finally, $\q_{k+1}$ is the best fit to these estimates in $\F$, i.e.,
\begin{equation*}
\q_{k+1} = \fit{\left\{\big((s^{(i)},a^{(i)}),\widehat{\q}_{k+1}(s^{(i)},a^{(i)})\big)\right\}_{i=1}^N}.
\end{equation*}
Each iteration of AMPI-Q requires $Nm$ samples, which is less than that for AMPI-V. However, it uses a hypothesis space on state-action pairs instead of states. Note that the fitted-Q iteration algorithm~\citep{Ernst:2005,Antos:2007} is a special case of AMPI-Q when $m=1$.


\subsection{Classification-Based MPI}

\begin{figure}[t]
\bookbox{
\begin{small}
\begin{algorithmic}
\STATE \textbf{Input:} Value function space $\F$, policy space $\polSpace$, state distribution $\mu$
\STATE \textbf{Initialize:} Let $\pi_1\in\polSpace$ be an arbitrary policy and $v_0\in\F$ an arbitrary value function
\FOR{$k = 1,2,\ldots$}
\STATE \textbf{\textit{\textbullet$\:$ Perform rollouts:}}
\STATE Construct the rollout set $\Data_k=\{s^{(i)}\}_{i=1}^n,\;s^{(i)}\stackrel{\text{iid}}{\sim}\mu$
\FORALL{states $s^{(i)}\in\Data_k$}
\STATE Perform a rollout 
  and return $\hv_{k}(s^{(i)})$
\ENDFOR
\STATE Construct the rollout set $\Data'_k=\{s^{(i)}\}_{i=1}^\nSamples,\;s^{(i)}\stackrel{\text{iid}}{\sim}\mu$
\FORALL{states $s^{(i)}\in\Data'_k$ and actions $a\in\action$}
\FOR{$j=1$ to $M$}
\STATE Perform a rollout and return $R_k^j(s^{(i)},a)$
\ENDFOR
\STATE $\hQ_k(s^{(i)},a) = \frac{1}{M} \sum_{j=1}^M R_k^j(s^{(i)},a)$
\ENDFOR
\STATE \textbf{\textit{\textbullet$\:$ Approximate value function:}}
\STATE $v_k \in \argmin\limits_{v \in \F} \widehat{\L}_{k}^{\F}(\widehat{\mu};v)$\hfill {\bf (regression)}
\STATE \textbf{\textit{\textbullet$\:$ Approximate greedy policy:}}
\STATE $\pi_{k+1} \in \argmin\limits_{\pol\in\polSpace}\widehat{\L}_{k}^{\,\Pi}(\widehat{\mu};\pi)$\hfill {\bf (classification)}
\ENDFOR
\end{algorithmic}
\end{small}}
\caption{The pseudo-code of the CBMPI algorithm.}\label{f:algorithm}
\end{figure}

The third AMPI algorithm presented in this paper, called classification-based MPI (CBMPI), uses an explicit representation for the policies $\pi_k$, in addition to the one used for value functions $v_k$. The idea is similar to the classification-based PI algorithms~\citep{Lagoudakis:2003b,Fern:2006,Lazaric:2010,Gabillon11CP} in which we search for the greedy policy in a policy space $\Pi$ (defined by a classifier) instead of computing it from the estimated value or action-value function (like in AMPI-V and AMPI-Q).

In order to describe CBMPI, we first rewrite the MPI formulation (Eqs.~\ref{eq:greedystep} and~\ref{eq:evaluationstep}) as
\begin{align}
v_{k} & =  (T_{\pi_{k}})^m v_{k-1} &&\mbox{(evaluation step)} \label{eq:greedystep2}\\
\pi_{k+1} & =  \greedy \big[ (T_{\pi_{k}})^m v_{k-1}\big] &&\mbox{(greedy step)} \label{eq:evaluationstep2}
\end{align}
Note that in the new formulation both $v_k$ and $\pi_{k+1}$ are functions of $(T_{\pi_{k}})^m v_{k-1}$. CBMPI is an approximate version of this new formulation. As described in Fig.~\ref{f:algorithm}, CBMPI begins with arbitrary initial policy $\pi_1 \in  \Pi $ and value function $v_0 \in \F$.\footnote{Note that the function space $\F$ and policy space $\Pi$ are automatically defined by the choice of the regressor and classifier, respectively.} At each iteration $k$, a new value function $v_k$ is built as the best approximation of the $m$-step Bellman operator $(T_{\pi_{k}} )^mv_{k-1}$ in $\F$ ({\em evaluation step}). This is done by solving a regression problem whose target function is $(T_{\pi_{k}} )^mv_{k-1}$.
%
%
To set up the regression problem, we build a rollout set $\D_k$ by sampling $n$ states i.i.d.~from a distribution $\mu$.\footnote{Here we used the same sampling distribution $\mu$ for both regressor and classifier, but in general different distributions may be used for these two components.} For each state $s^{(i)} \in \D_k$, we generate a rollout $\big(s^{(i)},a_0^{(i)},r_0^{(i)},s_1^{(i)},\ldots,a_{m-1}^{(i)},r_{m-1}^{(i)},s_m^{(i)}\big)$ of size $m$, where $a_t^{(i)}=\pi_{k}(s_t^{(i)})$, and $r_t^{(i)}$ and $s_{t+1}^{(i)}$ are the reward and next state induced by this choice of action. From this rollout, we compute an unbiased estimate $\widehat{v}_k(s^{(i)})$ of $\big[(T_{\pi_k})^mv_{k-1}\big](s^{(i)})$ as
\begin{equation}
\label{eq:rollout0}
\hV_k(s^{(i)}) =\sum^{m-1}_{t=0}\discount^t r_t^{(i)}+\discount^mv_{k-1}(s^{(i)}_m),
\end{equation}
and use it to build a training set $\big\{\big(s^{(i)},\hV_k(s^{(i)})\big)\big\}_{i=1}^n$. This training set is then used by the regressor to compute $v_k$ as an estimate of $(T_{\pi_{k}} )^mv_{k-1}$.

The {\em greedy step} at iteration $k$ computes the policy $\pi_{k+1}$ as the best approximation of $\G\big[(T_{\pi_k})^mv_{k-1}\big]$ by solving a cost-sensitive classification problem. From the definition of a greedy policy, if $\pi=\G\big[(T_{\pi_k})^mv_{k-1}\big]$, for each $s\in\mathcal{S}$, we have
\begin{equation}
\label{eq:greedy-CBMPI1}
\big[T_{\pi}(T_{\pi_k})^mv_{k-1}\big](s)=\max_{a\in\A}\big[T_a(T_{\pi_k})^mv_{k-1}\big](s).
\end{equation}
By defining $Q_k(s,a)=\big[T_a(T_{\pi_k})^mv_{k-1}\big](s)$, we may rewrite Eq.~\ref{eq:greedy-CBMPI1} as
\begin{equation}
\label{eq:greedy-CBMPI2}
Q_k\big(s,\pi(s)\big)=\max_{a\in\A}Q_k(s,a).
\end{equation}
The cost-sensitive error function used by CBMPI is of the form 
\begin{small}
\begin{align*}
\L_{\pi_k,v_{k-1}}^{\Pi}(\mu;\pol)=\int_{\state}\Big[\max_{a\in\action} \Qfun_k(s,a)-\Qfun_k\big(s,\pol(s)\big)\Big]\mu(ds).
\end{align*}
\end{small}
To simplify the notation we use $\L_{k}^{\Pi}$ instead of $\L_{\pi_k,v_{k-1}}^{\Pi}$. To set up this cost-sensitive classification problem, we build a rollout set $\D'_k$ by sampling $N$ states i.i.d.~from a distribution $\mu$. For each state $s^{(i)}\in\Data'_k$ and each action $a\in\action$, we build $M$ independent rollouts of size $m+1$, i.e.,\footnote{We may implement CBMPI more sample efficient by reusing the rollouts generated for the greedy step in the evaluation step. 
}
\begin{equation*}
\big(s^{(i)},a,r_0^{(i,j)},s_1^{(i,j)},a_1^{(i,j)},\ldots,a_m^{(i,j)},r_m^{(i,j)},s_{m+1}^{(i,j)}\big)_{j=1}^M,
\end{equation*}
where for $t\geq 1$, $a_t^{(i,j)}=\pi_k(s_t^{(i,j)})$, and $r_t^{(i,j)}$ and $s_{t+1}^{(i,j)}$ are the reward and next state induced by this choice of action. From these rollouts, we compute an unbiased estimate of $Q_k(s^{(i)},a)$ as $\hQ_k(s^{(i)},a)=\frac{1}{M}\sum_{j=1}^M R_k^j(s^{(i)},a)$ where
\begin{align*}
R_k^j(s^{(i)},a) = \sum_{t=0}^m\gamma^tr_t^{(i,j)}+\gamma^{m+1}v_{k-1}(s_{m+1}^{(i,j)}).
\end{align*}
Given the outcome of the rollouts, CBMPI uses a cost-sensitive classifier to return a policy $\pi_{k+1}$ that minimizes the following {\em empirical error}
\begin{small}
\begin{equation*}
\widehat{\L}_{k}^{\,\Pi}(\widehat{\mu};\pi) = \frac{1}{N} \sumSamples\Big[\max_{a\in\action} \hQ_k(s^{(i)},a) - \hQ_k\big(s^{(i)},\pol(s^{(i)})\big)\Big],
\end{equation*}
\end{small}
with the goal of minimizing the true error $\L_k^\Pi(\mu;\pi)$. 

Each iteration of CBMPI requires $nm+M|\A|N(m+1)$ (or $M|\A|N(m+1)$ in case we reuse the rollouts, see Footnote~3) transition samples. Note that when $m$ tends to $\infty$, we recover the DPI algorithm proposed and analyzed by~\citet{Lazaric:2010}.

\section{Error propagation}
\label{analysis}

In this section, we derive a general formulation for propagation of error through the iterations of an AMPI algorithm. The line of analysis for error propagation is different in VI and PI algorithms. VI analysis is based on the fact that this algorithm computes the fixed point of the Bellman optimality operator, and this operator is a $\gamma$-contraction in max-norm~\citep{Bertsekas:1996,Munos:2007}. On the other hand, it can be shown that the operator by which PI updates the value from one iteration to the next is not a contraction in max-norm in general. Unfortunately, we can show that the same property holds for MPI when it does not reduce to VI (i.e.,~$m>1$). 
\begin{proposition}\label{prop:contract}
{If $m>1$, there exists no norm for which the operator that MPI uses to update the values from one iteration to the next is a contraction.}
\end{proposition}
\begin{proof}
Consider a deterministic MDP with two states $\{s_1,s_2\}$, two actions $\{change,stay\}$, rewards $r(s_1)=0, r(s_2)=1$, and 
transitions $P_{ch}(s_2|s_1)=P_{ch}(s_1|s_2)=P_{st}(s_1|s_1)=P_{st}(s_2|s_2)=1$.
Consider the following two value functions $v=(\epsilon,0)$ and $v'=(0,\epsilon)$ with $\epsilon>0$. Their corresponding greedy policies are $\pi=(st,ch)$ and $\pi'=(ch,st)$, and the next iterates of $v$ and $v'$ can be computed as $(T_\pi)^mv=\begin{pmatrix} \gamma^m \epsilon \\ 1+\gamma^m \epsilon \end{pmatrix}$ and
$(T_{\pi'})^m v' = \begin{pmatrix} \frac{\gamma-\gamma^m}{1-\gamma}+\gamma^m \epsilon \\\frac{1-\gamma^m}{1-\gamma}+\gamma^m \epsilon \end{pmatrix}$. Thus, $(T_{\pi'})^m v'-(T_\pi)^m v=\begin{pmatrix} \frac{\gamma-\gamma^m}{1-\gamma} \\ \frac{\gamma-\gamma^m}{1-\gamma} \end{pmatrix}$ while $v'-v = \begin{pmatrix}-\epsilon \\ \epsilon\end{pmatrix}$. Since $\epsilon$ can be arbitrarily small, the norm of $(T_{\pi'})^m v' -(T_\pi)^m v$ can be arbitrarily larger than the norm of $v-v'$ as long as $m>1$. 
\end{proof}
We also know that the analysis of PI usually relies on the fact that the sequence of the generated values is non-decreasing~\citep{Bertsekas:1996,Munos:2003}. Unfortunately, it can easily be shown that for $m$ finite, the value functions generated by MPI may decrease (it suffices to take a very high initial value). It can be seen from what we just described and Proposition~\ref{prop:contract} that for $m\neq 1$ and $\infty$, MPI is neither contracting nor non-decreasing, and thus, a new line of proof is needed for the propagation of error in this algorithm.

To study error propagation in AMPI, we introduce an abstract algorithmic model that accounts for potential errors. AMPI starts with an arbitrary value $v_0$ and at each iteration $k\geq 1$ computes the greedy policy w.r.t.~$v_{k-1}$ with some error $\epsilon_k'$, called the {\em greedy step error}. Thus, we write the new policy $\pi_k$ as
\begin{equation}
\label{eq:abstractpi}
\pi_{k}  =  \widehat\greedy_{\epsilon'_k} v_{k-1}. 
\end{equation}
Eq.~\ref{eq:abstractpi} means that for any policy $\pi'$, 
\begin{equation*}
\label{eq:defapproxgreedy}
T_{\pi'} v_{k-1} \leq T_{\pi_k} v_{k-1} + \epsilon_k'.
\end{equation*}
AMPI then generates the new value function $v_k$ with some error $\epsilon_k$, called the {\em evaluation step error}
\begin{equation}
\label{eq:abstractv}
v_{k}  =  (T_{\pi_{k}})^m v_{k-1} + \epsilon_{k}. 
\end{equation}
Before showing how these two errors are propagated through the iterations of AMPI, let us first define them in the context of each of the algorithms presented in Section~\ref{algorithms} separately.

{\bf AMPI-V:} $\epsilon_k$ is the error in fitting the value function $v_k$. This error can be further decomposed into two parts: the one related to the approximation power of $\F$ and the one due to the finite number of samples/rollouts. $\epsilon'_k$ is the error due to using a finite number of samples $M$ for estimating the greedy actions.

{\bf AMPI-Q:} $\epsilon'_k=0$ and $\epsilon_k$ is the error in fitting the state-action value function $Q_k$.
   
{\bf CBMPI:} This algorithm iterates as follows:
\begin{align*}
v_{k} & =  (T_{\pi_{k}})^m v_{k-1} + \epsilon_k \\
\pi_{k+1} & =   \widehat\greedy_{\epsilon'_{k+1}} \left[ (T_{\pi_{k}})^m v_{k-1} \right] 
\end{align*}
Unfortunately, this does not exactly match with the model described in Eqs.~\ref{eq:abstractpi} and~\ref{eq:abstractv}. By introducing the auxiliary variable $w_k \= (T_{\pi_{k}})^m v_{k-1} $, we have $v_k = w_k + \epsilon_{k}$, and thus, we may write
\begin{equation}
\label{eq:def_error2_cbmpi}
\pi_{k+1} = \widehat\greedy_{\epsilon'_{k+1}}\left[ w_k \right].
\end{equation}
Using $v_{k-1}=w_{k-1}+\epsilon_{k-1}$, we have
\begin{align}
\label{eq:def_error_cbmpi}
w_k&=(T_{\pi_k})^m v_{k-1} = (T_{\pi_k})^m (w_{k-1}+\epsilon_{k-1}) \nonumber \\
&=(T_{\pi_k})^m w_{k-1}+ (\gamma P_{\pi_k})^m  \epsilon_{k-1}.
\end{align}
Now, Eqs.~\ref{eq:def_error2_cbmpi} and~\ref{eq:def_error_cbmpi} exactly match Eqs.~\ref{eq:abstractpi} and~\ref{eq:abstractv} by replacing $v_k$ with $w_k$ and $\epsilon_k$ with $(\gamma P_{\pi_k})^m \epsilon_{k-1} $. 

The rest of this section is devoted to show how the errors $\epsilon_k$ and $\epsilon'_k$ propagate through the iterations of an AMPI algorithm. We only outline the main arguments that will lead to the performance bound of Thm.~\ref{lpbound} and report most proofs in \citep{scherrer2012}. We follow the line of analysis developped by~\citet{scherrer2010}. The results are obtained using the following three quantities:

{\bf 1)} The distance between the optimal value function and the value before approximation at the $k^\text{th}$ iteration: $d_k\=v_*-(T_{\pi_k})^m v_{k-1}=v_*-(v_k-\epsilon_k)$.

{\bf 2)} The shift between the value before approximation and the value of the policy at the $k^\text{th}$ iteration: $s_k\=(T_{\pi_k})^m v_{k-1}-v_{\pi_k}=(v_k-\epsilon_k)-v_{\pi_k}$.

{\bf 3)} The Bellman residual at the $k^\text{th}$ iteration: $b_k\=v_k - T_{\pi_{k+1}} v_k.$

We are interested in finding an upper bound on the {\bf loss} $l_k \= v_*-v_{\pi_k} = d_k+s_k$. To do so, we will upper bound $d_k$ and $s_k$, which requires a bound on the Bellman residual $b_k$. More precisely, the core of our analysis is to prove the following point-wise inequalities for our three quantities of interest.
\begin{lemma}[Proof in \appdx{\ref{proof:3termsAMPI}}]
\label{lemma:core}
Let $k \ge 1$, $x_k\=(I-\gamma P_{\pi_k})\epsilon_k + \epsilon'_{k+1}$ and $y_k \=-\gamma P_{\pi_*}\epsilon_k+\epsilon'_{k+1}$. We have:
\begin{align*}
b_k &\leq (\gamma P_{\pi_k})^m b_{k-1} + x_k, \\
d_{k+1} &\leq \gamma P_{\pi_*}d_k + y_k + \sum_{j=1}^{m-1} (\gamma P_{\pi_{k+1}})^j b_k , \\
s_{k} &= (\gamma P_{\pi_{k}})^m (I-\gamma P_{\pi_{k}})^{-1}b_{k-1}.
\end{align*}
\end{lemma}
%
%
Since the stochastic kernels are non-negative, the bounds in Lemma~\ref{lemma:core} indicate that the loss $l_k$ will be bounded if the errors $\epsilon_k$ and $\epsilon'_k$ are controlled. In fact, if we define $\epsilon$ as a uniform upper-bound on the errors $|\epsilon_k|$ and $|\epsilon'_k|$, the first inequality in Lemma~\ref{lemma:core} implies that $b_k \leq O(\epsilon)$, and as a result, the second and third inequalities gives us $d_k \le O(\epsilon)$ and $s_k \le O(\epsilon)$. This means that the loss will also satisfy $l_k \le O(\epsilon)$.

Our bound for the loss $l_k$ is the result of careful expansion and combination of the three inequalities in Lemma~\ref{lemma:core}. Before we state this result, we introduce some notations that will ease our formulation.

\begin{definition}
\label{def:kernel-set}
For a positive integer $n$, we define $\setP_n$ as the set of transition kernels that are defined as follows:

{\bf 1)} for any set of $n$ policies $\{\pi_1,\ldots,\pi_n\}$, $(\gamma P_{\pi_1})(\gamma P_{\pi_2})\ldots (\gamma P_{\pi_n})\in\setP_n$,

{\bf 2)} for any $\alpha \in (0,1)$ and $(P_1,P_2)\in\setP_n\times\setP_n$, $\alpha P_1 + (1-\alpha) P_2 \in \setP_n$.

Furthermore, we use the somewhat abusive notation $\Gamma^n$ for denoting any element of $\setP_n$. For example, if we write a transition kernel $P$ as $P=\alpha_1\Gamma^i+\alpha_2\Gamma^j\Gamma^k=\alpha_1\Gamma^i+\alpha_2\Gamma^{j+k}$, it should be read as there exist $P_1\in\setP_i$, $P_2\in\setP_j$, $P_3\in\setP_k$, and $P_4\in\setP_{k+j}$ such that $P=\alpha_1P_1+\alpha_2P_2P_3=\alpha_1P_1+\alpha_2P_4$.
\end{definition}

Using the notation introduced in Definition~\ref{def:kernel-set}, we now derive a point-wise bound on the loss.
\begin{lemma}[Proof in \appdx{\ref{proof:lossAMPI}}]
\label{lemma:core2}
After $k$ iterations, the losses of AMPI-V and AMPI-Q satisfy
\begin{equation*}
\label{eq:pointwise}
l_k \leq 2\sum_{i=1}^{k-1}\sum_{j=i}^\infty \Gamma^j |\epsilon_{k-i}| + \sum_{i=0}^{k-1}\sum_{j=i}^\infty \Gamma^j |\epsilon'_{k-i}| + h(k),
\end{equation*}
while the loss of CBMPI satisfies
\begin{equation*}
l_k \leq 2\sum_{i=1}^{k-2}\sum_{j=i+m}^\infty \Gamma^{j} |\epsilon_{k-i-1}| + \sum_{i=0}^{k-1}\sum_{j=i}^\infty \Gamma^j |\epsilon'_{k-i}| + h(k),
\end{equation*}
where $h(k)\= 2 \sum_{j=k}^\infty \Gamma^j |d_0|$ or $h(k)\=2\sum_{j=k}^\infty \Gamma^j |b_0|$.
\end{lemma}
%
%
\begin{remark}
\label{rem:unifyc}
A close look at the existing point-wise error bounds for AVI~\citep[Lemma~4.1]{Munos:2007} and API~\citep[Corollary 10]{Munos:2003} shows that they do not consider error in the greedy step (i.e.,~$\epsilon'_k=0$) and that they have the following form:
\begin{equation*}
{\lim\sup}_{k \rightarrow \infty} l_k \le 2\;{\lim\sup}_{k \rightarrow \infty} \sum_{i=1}^{k-1} \sum_{j=i}^\infty \pp{j}|\epsilon_{k-i}|.
\end{equation*}
This indicates that the bound in Lemma~\ref{lemma:core2} not only unifies the analysis of AVI and API, but it generalizes them to the case of error in the greedy step and to a finite horizon $k$. Moreover, our bound suggests that the way the errors are propagated in the whole family of algorithms VI/PI/MPI does not depend on $m$ at the level of the abstraction suggested by Definition~\ref{def:kernel-set}.\footnote{Note however that the dependence on $m$ will reappear if we make explicit what is hidden in the terms $\Gamma^j$.}
\end{remark}
The next step is to show how the point-wise bound of Lemma~\ref{lemma:core2} can turn to a bound in weighted $L_p$-norm, which for any function $f:\S \rightarrow \R$ and any distribution $\mu$ on $\S$ is defined as $\|f\|_{p,\mu}\=\left(\int |f(x)|^p \mu(dx)\right)^{1/p}$. \citet{Munos:2003,Munos:2007,Munos:2008}, and the recent work of~\citet{Farahmand:2010}, which provides the most refined bounds for API and AVI, show how to do this process
through quantities, called {\em concentrability coefficients}, that measure how a distribution over states may concentrate through the dynamics of the MDP. We now state a lemma that generalizes the analysis of~\citet{Farahmand:2010} to a larger class of concentrability coefficients. We will discuss the potential advantage of this new class in Remark~\ref{rem:qqp}. We will also show through the proofs of Thms.~\ref{lpbound} and~\ref{alternativebound}, how the result of Lemma~\ref{lemma:fromctolp} provides us with a flexible tool for turning point-wise bounds into $L_p$-norm bounds. Thm.~\ref{alternativebound} in \appdx{\ref{proof:lpbound}} provides an alternative bound for the loss of AMPI, which in analogy with the results of~\citet{Farahmand:2010} shows that the last iterations have the highest impact on the loss (the influence exponentially decreases towards the initial iterations).  
\begin{lemma}[Proof in \appdx{\ref{proof:concenter}}]
\label{lemma:fromctolp}
Let $\cal I$ and $({\cal J}_i)_{i \in {\cal I}}$ be sets of positive integers, $\{{\cal I}_1,\dots,{\cal I}_n\}$ be a partition of $\cal I$, and $f$ and $(g_i)_{i \in {\cal I}}$ be functions satisfying 
\begin{equation*}
|f| \le \sum_{i \in \cal{I}} \sum_{j \in {\cal J}_{i}} \pp{j} |g_i| = \sum_{l=1}^n\sum_{i \in {\cal{I}}_l} \sum_{j \in {\cal J}_i} \pp{j} |g_i|.
\end{equation*}
Then for all $p$, $q$ and $q'$ such that $\frac{1}{q}+\frac{1}{q'}=1$, and for all distributions $\rho$ and $\mu$, we have
\begin{equation*}
\|f\|_{p,\rho} \le \sum_{l=1}^n \big({\cal C}_{q}(l)\big)^{1/p} \sup_{i \in {\cal I}_l}  \|g_{i}\|_{pq',\mu} \sum_{i \in {\cal I}_l} \sum_{j \in {\cal J}_{i}} \gamma^{j}, 
\end{equation*}
with the following concentrability coefficients
\begin{small}
\begin{equation*}
{\cal C}_{q}(l) \= \frac{\sum_{i \in {\cal I}_l} \sum_{j \in {\cal J}_{i}} \gamma^{j} c_q(j) }{ \sum_{i \in {\cal I}_l} \sum_{j \in {\cal J}_{i}} \gamma^{j}},
\end{equation*}
\end{small}
%
with the Radon-Nikodym derivative based quantity 
\begin{small}
\begin{equation}
c_q(j) \= \max_{\pi_1,\cdots,\pi_j}\left\| \frac{d (\rho P_{\pi_1}P_{\pi_2}\cdots P_{\pi_j})}{d \mu} \right\|_{q,\mu} \label{eq:defconcentrability}
\end{equation}

\end{small}
\end{lemma}
%
%
We now derive a $L_p$-norm bound for the loss of the AMPI algorithm by applying Lemma~\ref{lemma:fromctolp} to the point-wise bound of Lemma~\ref{lemma:core2}.
\begin{theorem}[Proof in \appdx{\ref{proof:lpbound}}]
\label{lpbound} 
Let $\rho$ and $\mu$ be distributions over states. Let $p$, $q$, and $q'$ be such that $\frac{1}{q}+\frac{1}{q'}=1$. After $k$ iterations, the loss of AMPI satisfies
\begin{small}
\begin{align}
\|l_k \|_{p,\rho} &\leq \frac{2 (\gamma-\gamma^k) \left( {\cal C}_{q}^{1,k,0} \right)^{\frac{1}{p}}}{(1-\gamma)^2}  \sup_{1 \le j \le k-1}\|\epsilon_j\|_{pq',\mu} \label{eq:thm1} \\ 
&+\frac{(1-\gamma^k)\left( {\cal C}_{q}^{0,k,0} \right)^{\frac{1}{p}}}{(1-\gamma)^2} \sup_{1 \le j \le k}\|\epsilon'_j\|_{pq',\mu}  + g(k), \nonumber
\end{align}
\end{small}
while the loss of CBMPI satisfies
\begin{small}
\begin{align}
\|l_k \|_{p,\rho} \leq \frac{2\gamma^m (\gamma-\gamma^{k-1}) \left( {\cal C}_{q}^{2,k,m} \right)^{\frac{1}{p}}}{(1-\gamma)^2}  \sup_{1 \le j \le k-2}\|\epsilon_j\|_{pq',\mu} \label{eq:thm1b}\\
+\frac{(1-\gamma^k)\left( {\cal C}_{q}^{1,k,0} \right)^{\frac{1}{p}}}{(1-\gamma)^2} \sup_{1 \le j \le k}\|\epsilon'_j\|_{pq',\mu}  + g(k), \nonumber
\end{align}
\end{small}
where for all $q$, $l$, $k$ and $d$, the concentrability coefficients ${\cal C}_{q}^{l,k,d}$ are defined as
\begin{small}
\begin{equation*}
{\cal C}_{q}^{l,k,d} \= \frac{(1-\gamma)^2}{\gamma^{l}-\gamma^{k}} \sum_{i=l}^{k-1}  \sum_{j=i}^\infty \gamma^{j} c_{q}(j+d),
\end{equation*}
\end{small}
with $c_q(j)$ given by Eq.~\ref{eq:defconcentrability}, and $g(k)$ is defined as $g(k) \=\frac{2\gamma^k}{1-\gamma}\left({\cal C}^{k,k+1}_{q}\right)^{\frac{1}{p}}\min\big(\|d_0\|_{pq',\mu},\|b_0\|_{pq',\mu}\big)$. 
\end{theorem}
%
%
\begin{remark}
\label{recoveraviapi}
When $p$ tends to infinity, the first bound of Thm.~\ref{lpbound} reduces to

\vspace{-0.3in}
\begin{small}
\begin{align}
\|l_k \|_{\infty} &\le \frac{2 (\gamma-\gamma^k)}{(1-\gamma)^2}  \sup_{1 \le j \le k-1}\|\epsilon_j\|_{\infty}+ \frac{1-\gamma^k}{(1-\gamma)^2}\sup_{1 \le j \le k}\|\epsilon'_j\|_{\infty} \nonumber \\
&+\frac{2\gamma^k}{1-\gamma} \min (\|d_0\|_{\infty},\|b_0\|_{\infty}). \label{eq:remark2}
\end{align}
\end{small}
\vspace{-0.2in}

When $k$ goes to infinity, Eq.~\ref{eq:remark2} gives us a generalization of the API ($m=\infty$) bound of~\citet[Prop.~6.2]{Bertsekas:1996}, i.e.,

\vspace{-0.1in}
\begin{small}
\begin{equation*}
\limsup_{k \rightarrow \infty} \|l_k \|_{\infty} \le \frac{2 \gamma \sup_{j}\|\epsilon_j\|_{\infty}  + \sup_{j}\|\epsilon'_j\|_{\infty}}{(1-\gamma)^2}.
\end{equation*}
\end{small}
\vspace{-0.2in}

Moreover, since our point-wise analysis generalizes those of API and AVI (as noted in Remark~\ref{rem:unifyc}), the $L_p$-bound of Eq.~\ref{eq:thm1} unifies and generalizes those for API~\cite{Munos:2003} and AVI~\cite{Munos:2007}.
\end{remark}
\begin{remark}
\citet{canbolat2012} recently (and independently) developped an analysis of an approximate form of MPI. Also, as mentionned, the proof technique that we used is based on that of \citet{scherrer2010}.
While \citet{canbolat2012} only consider the error in the greedy step and \citet{scherrer2010} that in the value update, our work is more general in that we consider \emph{both} sources of error -- this is required for the analysis of CBMPI. \citet{scherrer2010} and \citet{canbolat2012} provide bounds when the errors are controlled in max-norm, while we consider the more general $L_p$-norm.
At a more technical level, Th.~2 in \citep{canbolat2012} bounds the norm of the distance $v_*-v_k$ while we bound the loss $v_*-v_{\pi_k}$. If we derive a bound on the loss (using e.g., Th.~1 in \citep{canbolat2012}), this leads to a bound on the loss  that is looser than ours. In particular, this does not allow to recover the standard bounds for AVI/API, as we managed to (c.f. Remark~\ref{recoveraviapi}).
\end{remark}
\begin{remark}
\label{rem:qqp} 
We can balance the influence of the concentrability coefficients (the bigger the $q$, the higher the influence) and the difficulty of controlling the errors (the bigger the $q'$, the greater the difficulty in controlling the $L_{pq'}$-norms) by tuning the parameters $q$ and $q'$, given the condition that $\frac{1}{q}+\frac{1}{q'}=1$. This potential leverage is an improvement over the existing bounds and concentrability results that only consider specific values of these two parameters: $q=\infty$ and $q'=1$ in \citet{Munos:2007,Munos:2008}, and $q=q'=2$ in \citet{Farahmand:2010}.
\end{remark}
\begin{remark}\label{r:m}
For CBMPI, the parameter $m$ controls the influence of the value function approximator, cancelling it out in the limit when $m$ tends to infinity (see Eq.~\ref{eq:thm1b}). Assuming a fixed budget of sample transitions, increasing $m$ reduces the number of rollouts used by the classifier and thus worsens its quality; in such a situation, $m$ allows to make a trade-off between the estimation error of the classifier and the overall value function approximation. 
\end{remark}
%


\section{Finite-Sample Analysis of CBMPI}
\label{cbmpi}

In this section, we focus on CBMPI and detail the possible form of the error terms that appear in the bound of Thm.~\ref{lpbound}. 
We select CBMPI among the proposed algorithms because its analysis is more general than the others as we need to bound both greedy and evaluation step errors (in some norm), 
and also because it displays an interesting influence of the parameter $m$ (\cf Remark~\ref{r:m}).
%
%
We first provide a bound on the {\em greedy step error}. From the definition of $\epsilon'_k$ for CBMPI (Eq.~\ref{eq:def_error2_cbmpi}) and the description of the greedy step in CBMPI, we can easily observe that $\|\epsilon'_k\|_{1,\mu}=\L_{k-1}^{\Pi}(\mu;\pi_k)$.
\begin{lemma}[Proof in \appdx{\ref{proof:actorCBMPI}}]
\label{l:actorCBMPI} 
Let $\polSpace$ be a policy space with finite VC-dimension $h=VC(\polSpace)$ and $\mu$ be a distribution over the state space $\S$. Let $N$ be the number of states in $\D'_{k-1}$ drawn i.i.d.~from $\mu$, $M$ be the number of rollouts per state-action pair used in the estimation of $\hQ_{k-1}$, and $\pi_k = \argmin_{\pi\in\Pi}\widehat{\L}_{k-1}^\Pi(\widehat{\mu},\pi)$ be the policy computed at iteration $k-1$ of CBMPI. Then, for any $\delta>0$, we have 
\begin{align*}
\|\epsilon'_k\|_{1,\mu}=\L_{k-1}^{\Pi}(\mu;\pi_k) \leq \inf_{\pi\in\Pi}\L^\Pi_{k-1}(\mu;\pi) + 2(\epsilon'_1+\epsilon'_2),
\end{align*}
with probability at least $1-\delta$, where
\begin{small}
\begin{align*}
\epsilon'_1(N,\delta) &= 16\Qmax\sqrt{\frac{2}{N}\big(h\log\frac{eN}{h}+\log\frac{32}{\delta}\big)}\;, \\
\epsilon'_2(N,M,\delta) &= 8\Qmax\sqrt{\frac{2}{MN}\big(h\log\frac{eMN}{h} + \log\frac{32}{\delta}\big)}\;.
\end{align*}
\end{small}
\end{lemma}
%
%
We now consider the {\em evaluation step error}. The evaluation step at iteration $k$ of CBMPI is a regression problem with the target $(T_{\pi_{k}})^mv_{k-1}$ and a training set $\big\{\big(s^{(i)},\hv_{k}(s^{(i)})\big)\big\}_{i=1}^n$ in which the states $s^{(i)}$ are i.i.d.~samples from $\mu$ and $\hv_{k}(s^{(i)})$ are unbiased estimates of the target computed according to Eq.~\ref{eq:rollout0}. Different function spaces $\F$ (linear or non-linear) may be used to approximate $(T_{\pi_{k}})^mv_{k-1}$. Here we consider a linear architecture with parameters $\alpha\in\Re^d$ and bounded (by $L$) basis functions $\{\varphi_j\}_{j=1}^d,\;\|\varphi_j\|_\infty\leq L$. We denote by $\phi:\X\rightarrow\Re^d,\;\phi(\cdot)=\big(\varphi_1(\cdot),\ldots,\varphi_d(\cdot)\big)^\top$ the feature vector, and by $\F$ the linear function space spanned by the features $\varphi_j$, i.e.,~$\F=\{f_\alpha(\cdot)=\phi(\cdot)^\top\alpha:\alpha\in\Re^d\}$. Now if we define $v_{k}$ as the truncation (by $\Vmax$) of the solution of the above linear regression problem, we may bound the {\em evaluation step error} using the following lemma.
\begin{lemma}[Proof in \appdx{\ref{proof:criticCBMPI}}]
\label{l:criticCBMPI} 
Consider the linear regression setting described above, then we have
\begin{align*}
\|\epsilon_{k}\|_{2,\mu} \leq 4\inf_{f\in\F}\|(T_{\pi_{k}})^mv_{k-1} - f\|_{2,\mu} + \epsilon_1 + \epsilon_2,
\end{align*}
with probability at least $1-\delta$, where
\begin{small}
\begin{align*}
\epsilon_1(n,\delta) &= 32\Vmax\sqrt{\frac{2}{n}\log\Big(\frac{27(12e^2n)^{2(d+1)}}{\delta}\Big)}\;, \\
\epsilon_2(n,\delta) &= 24\Big(\Vmax+\|\alpha_*\|_2\cdot\sup_x\|\phi(x)\|_2\Big)\sqrt{\frac{2}{n}\log\frac{9}{\delta}}\;,
\end{align*}
\end{small}
and $\alpha_*$ is such that $f_{\alpha_*}$ is the best approximation (w.r.t.~$\mu$) of the target function $(T_{\pi_{k}})^mv_{k-1}$ in $\F$.
\end{lemma}

From Lemmas~\ref{l:actorCBMPI} and~\ref{l:criticCBMPI}, we have bounds on $\|\epsilon'_k\|_{1,\mu}$ and $\|\epsilon_k\|_{1,\mu} \le \|\epsilon_k\|_{2,\mu}$. By a union bound argument, we thus control the r.h.s of Eq.~\ref{eq:thm1b}  in $L_1$ norm. In the context of Th.~\ref{lpbound}, this means $p=1$, $q'=1$ and $q=\infty$, and we have the following bound for CBMPI:
\begin{theorem}
Let $d'=\sup_{g \in \F, \pi'} \inf_{\pi\in\Pi}\L^\Pi_{\pi',g}(\mu;\pi)$ and $d_m=\sup_{g \in \F, \pi} \inf_{f \in F} \|(T_{\pi})^m g - f\|_{2,\mu}$. 
With the notations of Th.~\ref{lpbound} and Lemmas~\ref{l:actorCBMPI}-\ref{l:criticCBMPI}, after $k$ iterations, and with probability $1-\delta$, the expected loss $\E_\mu[l_k]=\|l_k\|_{1,\mu}$ of CBMPI is bounded by
\begin{small}
\begin{align}
& \frac{2\gamma^m (\gamma-\gamma^{k-1}) {{\cal C}_{\infty}^{2,k,m}}}{(1-\gamma)^2}  \left( d_m + \epsilon_1(n,\frac{\delta}{2k}) + \epsilon_2(n,\frac{\delta}{2k}) \right) \nonumber \\
&+\frac{(1-\gamma^k) {{\cal C}_{\infty}^{1,k,0}}}{(1-\gamma)^2} \left( d' + \epsilon'_1(N,\frac{\delta}{2k}) + \epsilon'_2(N,M,\frac{\delta}{2k}) \right)  + g(k). \nonumber
\end{align}
\end{small}
\end{theorem}
\begin{remark}
This result leads to a quantitative version of Remark~\ref{r:m}. Assume that we have a fixed budget for the actor and the critic $B=nm=NM|A|m$. Then, up to constants and logarithmic factors, the bound has the form
$
\|l_k\|_{1,\mu} \le O\left(  \gamma^m \left(d_m+\sqrt{\frac{m}{B}}\right) + d' + \sqrt{\frac{M|A|m}{B}}\right).
$
It shows the trade-off in the tuning of $m$: a big $m$ can make the influence of the overall (approximation and estimation) value error small, but that of the estimation error of the classifier bigger.
\end{remark}

\section{Summary and Extensions}

In this paper, we studied a DP algorithm, called modified policy iteration (MPI), that despite its generality that contains the celebrated policy and value iteration methods, has not been thoroughly investigated in the literature. We proposed three approximate MPI (AMPI) algorithms that are extensions of the well-known ADP algorithms: fitted-value iteration, fitted-Q iteration, and classification-based policy iteration. We reported an error propagation analysis for AMPI that unifies those for approximate policy and value iteration. We also provided a finite-sample analysis for the classification-based implementation of AMPI (CBMPI), whose analysis is more general than the other presented AMPI methods. Our results indicate that the parameter of MPI allows us to control the balance of errors (in value function approximation and estimation of the greedy policy) in the final performance of CBMPI. Although AMPI generalizes the existing AVI and classification-based API algorithms, additional experimental work and careful theoretical analysis are required to obtain a better understanding of the behaviour of its different implementations and their relation to the competitive methods. Extension of CBMPI to problems with continuous action space is another interesting direction to pursue.

\begin{small}
\bibliography{biblio}
\bibliographystyle{icml2012}
\end{small}


\newpage
\onecolumn

\begin{center}
\Large{Supplementary Material for\\
Approximate Modified Policy Iteration}
\end{center}

\appendix{


\section{Proof of Lemma~\ref{lemma:core}}
\label{proof:3termsAMPI}

Before we start, we recall the following definitions:
\begin{equation*}
b_k = v_k - T_{\pi_{k+1}} v_k, \quad\quad\quad d_k = v_*-(T_{\pi_k})^m v_{k-1} = v_*-(v_k-\epsilon_k), \quad\quad\quad s_k = (T_{\pi_k})^m v_{k-1} - v_{\pi_k} = (v_k-\epsilon_k)-v_{\pi_k}. 
\end{equation*}

\paragraph{Bounding $b_k$}
\begin{align}
b_k &= v_k - T_{\pi_{k+1}} v_k= v_k-T_{\pi_k}v_k + T_{\pi_k}v_k-T_{\pi_{k+1}} v_k \stackrel{(a)}{\leq} v_k-T_{\pi_k}v_k + \epsilon'_{k+1} \nonumber \\
&= v_k-\epsilon_k-T_{\pi_k}v_k+\gamma P_{\pi_k}\epsilon_k+\epsilon_k-\gamma P_{\pi_k}\epsilon_k+\epsilon'_{k+1} \stackrel{(b)}{=} v_k-\epsilon_k-T_{\pi_k}(v_k-\epsilon_k)+(I-\gamma P_{\pi_k})\epsilon_k+\epsilon'_{k+1}. \label{eq:b1}
\end{align}
Using the definition of $x_k$, i.e.,
\begin{equation}
\label{eq:defx}
x_k\=(I-\gamma P_{\pi_k})\epsilon_k + \epsilon'_{k+1},
\end{equation}
we may write Eq.~\eqref{eq:b1} as
\begin{align}
b_k  &\leq v_k-\epsilon_k-T_{\pi_k}(v_k-\epsilon_k)+x_k \stackrel{(c)}{=} (T_{\pi_k})^m v_{k-1}-T_{\pi_k}(T_{\pi_k})^{m}  v_{k-1} + x_k =  (T_{\pi_k})^m v_{k-1}-(T_{\pi_k})^{m} (T_{\pi_k} v_{k-1}) + x_k \nonumber \\
&= (\gamma P_{\pi_k})^m (v_{k-1}-T_{\pi_k}v_{k-1}) + x_k = (\gamma P_{\pi_k})^m b_{k-1} + x_k. \label{eq:b2}
\end{align}

{\bf (a)} From the definition of $\epsilon'_{k+1}$, we have $\forall \pi'\;\;T_{\pi'}v_k\leq T_{\pi_{k+1}}v_k+\epsilon'_{k+1}$, thus this inequality holds also for $\pi'=\pi_k$. 

{\bf (b)} This step is due to the fact that for every $v$ and $v'$, we have $T_{\pi_k}(v+v')=T_{\pi_k}v + \gamma P_{\pi_k} v'$.

{\bf (c)} This is from the definition of $\epsilon_k$, i.e., $v_k=(T_{\pi_k})^m v_{k-1}+\epsilon_k$.


\paragraph{Bounding $d_k$}

\begin{align}
d_{k+1}& = v_* - (T_{\pi_{k+1}})^m v_k = T_{\pi_*}v_*-T_{\pi_*}v_k+T_{\pi_*}v_k-T_{\pi_{k+1}}v_k+T_{\pi_{k+1}}v_k-(T_{\pi_{k+1}})^m v_k \nonumber \\
&\stackrel{(a)}{\leq} \gamma P_{\pi_*}(v_*-v_k)+\epsilon'_{k+1} + g_{k+1} = \gamma P_{\pi_*}(v_*-v_k) + \gamma P_{\pi_*}\epsilon_k - \gamma P_{\pi_*}\epsilon_k + \epsilon'_{k+1} + g_{k+1} \nonumber \\
&\stackrel{(b)}{=}\gamma P_{\pi_*}\big(v_*-(v_k-\epsilon_k)\big) + y_k + g_{k+1} = \gamma P_{\pi_*}d_k + y_k + g_{k+1} \stackrel{(c)}{=} \gamma P_{\pi_*}d_k + y_k + \sum_{j=1}^{m-1} (\gamma P_{\pi_{k+1}})^j b_k. \label{eq:d1}
\end{align}

{\bf (a)} This step is from the definition of $\epsilon'_{k+1}$ (see step {\bf (a)} in bounding $b_k$) and by defining $g_{k+1}$ as follows:
\begin{equation}
\label{eq:defQ}
g_{k+1}\=T_{\pi_{k+1}}v_k-(T_{\pi_{k+1}})^m v_k.
\end{equation}
{\bf (b)} This is from the definition of $y_k$, i.e.,
\begin{equation}
\label{eq:defy}
y_k \=-\gamma P_{\pi_*}\epsilon_k+\epsilon'_{k+1}.
\end{equation}
{\bf (c)} This step comes from rewriting $g_{k+1}$ as 
\begin{align}
g_{k+1}&=T_{\pi_{k+1}}v_k-(T_{\pi_{k+1}})^m v_k = \sum_{j=1}^{m-1}\big[(T_{\pi_{k+1}})^j v_k -  (T_{\pi_{k+1}})^{j+1}v_k\big] = \sum_{j=1}^{m-1} \big[(T_{\pi_{k+1}})^j v_k -  (T_{\pi_{k+1}})^{j}(T_{\pi_{k+1}} v_k)\big] \nonumber \\
&= \sum_{j=1}^{m-1} (\gamma P_{\pi_{k+1}})^j(v_k-T_{\pi_{k+1}}v_k) = \sum_{j=1}^{m-1} (\gamma P_{\pi_{k+1}})^jb_k. \label{eq:g1}
\end{align}


\paragraph{Bounding $s_k$}

With some slight abuse of notation, we have 
$$v_{\pi_{k}}=(T_{\pi_{k}})^\infty v_k$$
and thus:

\begin{align}
s_k &= (T_{\pi_{k}})^m v_{k-1} - v_{\pi_{k}} \stackrel{(a)}{=} (T_{\pi_{k}})^m v_{k-1} - (T_{\pi_{k}})^\infty v_{k-1} = (T_{\pi_{k}})^m v_{k-1} - (T_{\pi_{k}})^m(T_{\pi_{k}})^\infty v_{k-1} \nonumber \\
&= (\gamma P_{\pi_{k}})^m \big(v_{k-1}-(T_{\pi_{k}})^\infty v_{k-1}\big) = (\gamma P_{\pi_{k}})^m \sum_{j=0}^\infty \big[(T_{\pi_{k}})^j v_{k-1} -  (T_{\pi_{k}})^{j+1}v_{k-1}\big] \nonumber \\
&= (\gamma P_{\pi_{k}})^m \big(\sum_{j=0}^\infty \big[(T_{\pi_{k}})^j v_{k-1} -  (T_{\pi_{k}})^jT_{\pi_{k}}v_{k-1}\big] = (\gamma P_{\pi_{k}})^m \Big(\sum_{j=0}^\infty (\gamma P_{\pi_k})^j\Big) (v_{k-1} - T_{\pi_{k}}v_{k-1}) \nonumber \\
&= (\gamma P_{\pi_{k}})^m (I-\gamma P_{\pi_{k}})^{-1}(v_{k-1}-T_{\pi_{k}}v_{k-1}) = (\gamma P_{\pi_{k}})^m (I-\gamma P_{\pi_{k}})^{-1} b_k \label{eq:s1}.
\end{align}
{\bf (a)} For any $v$, we have $v_{\pi_{k}}=(T_{\pi_{k}})^\infty v$. This step follows by setting $v=v_{k-1}$, i.e.,~$v_{\pi_{k}}=(T_{\pi_{k}})^\infty v_{k-1}$.


\newpage
\section{Proof of Lemma~\ref{lemma:core2}}
\label{proof:lossAMPI}

We begin by focusing our analysis on AMPI.
Here we are interested in bounding the loss $l_k = v_*-v_{\pi_k} = d_k+s_k$. 

By induction, from Eqs.~\eqref{eq:b2} and~\eqref{eq:d1}, we obtain
\begin{equation}
\label{eq:indb}
b_k \leq \sum_{i=1}^{k}\Gamma^{m(k-i)}x_i + \Gamma^{mk}b_{0},
\end{equation}
\begin{equation}
\label{eq:indd}
d_k \leq \sum_{j=0}^{k-1} \Gamma^{k-1-j}\Big(y_j+\sum_{l=1}^{m-1} \Gamma^l b_j\Big) + \Gamma^k d_0.
\end{equation}
in which we have used the notation introduced in Definition~\ref{def:kernel-set}. In Eq.~\eqref{eq:indd}, we also used the fact that from Eq.~\eqref{eq:g1}, we may write $g_{k+1}=\sum_{j=1}^{m-1} \Gamma^j b_k$. Moreover, we may rewrite Eq.~\eqref{eq:s1} as 
\begin{equation}
\label{eq:inds}
s_k= \pp{m}\sum_{j=0}^\infty \pp{j} b_{k-1} = \sum_{j=0}^\infty \pp{m+j} b_{k-1}.
\end{equation}
%


\paragraph{Bounding $l_k$}

From Eqs.~\eqref{eq:indb} and~\eqref{eq:indd}, we may write
\begin{align}
d_k &\leq \sum_{j=0}^{k-1} \Gamma^{k-1-j} \left(y_j + \sum_{l=1}^{m-1} \Gamma^l \Big(\sum_{i=1}^j \Gamma^{m(j-i)}x_i + \Gamma^{mj}b_0\Big)\right) + \Gamma^kd_0 \nonumber \\
&= \sum_{i=1}^{k} \Gamma^{i-1} y_{k-i} + \sum_{j=0}^{k-1} \sum_{l=1}^{m-1}\sum_{i=1}^j \Gamma^{k-1-j+l+m(j-i)}x_i  + z_k, \label{eq:l1}
\end{align}
where we used the following definition
\begin{align*}
z_k \= \sum_{j=0}^{k-1} \sum_{l=1}^{m-1} \Gamma^{k-1+l+j(m-1)}b_0 + \Gamma^kd_0= \sum_{i=k}^{mk-1} \Gamma^ib_0+ \Gamma^kd_0.
\end{align*}
The triple sum involved in Eq.~\eqref{eq:l1} may be written as
\begin{align}
\sum_{j=0}^{k-1} \sum_{l=1}^{m-1}\sum_{i=1}^j \Gamma^{k-1-j+l+m(j-i)}x_i &= \sum_{i=1}^{k-1}  \sum_{j=i}^{k-1} \sum_{l=1}^{m-1} \Gamma^{k-1+l+j(m-1)-mi}x_i = \sum_{i=1}^{k-1} \sum_{j=mi+k-i}^{mk-1} \pp{j-mi}x_i \nonumber \\
&= \sum_{i=1}^{k-1} \sum_{j=k-i}^{m(k-i)-1} \Gamma^jx_i = \sum_{i=1}^{k-1} \sum_{j=i}^{mi-1} \Gamma^jx_{k-i}. \label{eq:l2}
\end{align}
Using Eq.~\eqref{eq:l2}, we may write Eq.~\eqref{eq:l1} as
\begin{equation}
\label{eq:finald}
d_k \le \sum_{i=1}^{k} \Gamma^{i-1} y_{k-i} + \sum_{i=1}^{k-1} \sum_{j=i}^{mi-1} \Gamma^jx_{k-i} + z_k.
\end{equation}
Similarly, from Eqs.~\eqref{eq:inds} and~\eqref{eq:indb}, we have
\begin{align}
s_k &\leq  \sum_{j=0}^\infty \Gamma^{m+j} \Big(\sum_{i=1}^{k-1} \Gamma^{m(k-1-i)}x_i + \Gamma^{m(k-1)}b_0\Big) = \sum_{j=0}^\infty \Big( \sum_{i=1}^{k-1} \Gamma^{m+j+m(k-1-i)}x_i + \Gamma^{m+j+m(k-1)}b_0 \Big) \nonumber \\
&= \sum_{i=1}^{k-1} \sum_{j=0}^\infty \Gamma^{j+m(k-i)}x_i + \sum_{j=0}^{\infty} \Gamma^{j+mk}b_0 = \sum_{i=1}^{k-1} \sum_{j=0}^\infty \Gamma^{j+mi}x_{k-i} + \sum_{j=mk}^\infty \Gamma^jb_0 = \sum_{i=1}^{k-1} \sum_{j=mi}^\infty \Gamma^jx_{k-i} +  z'_k, \label{eq:finals}
\end{align}
where we used the following definition 
\begin{equation*}
z'_k \= \sum_{j=mk}^\infty \Gamma^jb_0 .
\end{equation*}
Finally, using the bounds in Eqs.~\eqref{eq:finald} and~\eqref{eq:finals}, we obtain the following bound on the loss
{\begin{align}
l_k &\leq d_k + s_k \leq \sum_{i=1}^{k} \Gamma^{i-1} y_{k-i} + \sum_{i=1}^{k-1} \Big(\sum_{j=i}^{mi-1} \Gamma^j + \sum_{j=mi}^\infty \Gamma^j\Big) x_{k-i} + z_k + z'_k \nonumber \\
& =  \sum_{i=1}^{k} \Gamma^{i-1} y_{k-i} + \sum_{i=1}^{k-1} \sum_{j=i}^\infty \Gamma^j x_{k-i} + \eta_k, \label{eq:propagationxy}
\end{align}}
where we used the following definition
\begin{equation}
\eta_k \= z_k+z'_k = \sum_{j=k}^{\infty} \Gamma^j b_0 + \Gamma^kd_0. \label{eq:defetak}
\end{equation}
Note that we have the following relation between $b_0$ and $d_0$ 
\begin{equation}
b_0 = v_0 - T_{\pi_1} v_0 = v_0 - v_* + T_{\pi_*}v_* - T_{\pi_*}v_0 + T_{\pi_*}v_0 - T_{\pi_1} v_0 \leq (I - \gamma P_{\pi_*}) (-d_0) + \epsilon'_1, \label{eq:relb0d0}
\end{equation}
In Eq.~\eqref{eq:relb0d0}, we used the fact that $v_*=T_{\pi_*}v_*$, $\epsilon_0=0$, and $T_{\pi_*}v_0 - T_{\pi_1}v_0\leq\epsilon'_1$ (this is because the policy $\pi_1$ is $\epsilon'_1$-greedy w.r.t.~$v_0$). As a result, we may write $|\eta_k|$ either as
\begin{equation}
|\eta_k| \leq \sum_{j=k}^{\infty} \Gamma^j \big[(I-\gamma P_{\pi_*})|d_0|+|\epsilon'_1|\big] + \Gamma^k|d_0| \leq \sum_{j=k}^{\infty} \Gamma^j \big[(I+\Gamma^1)|d_0|+|\epsilon'_1|\big] + \Gamma^k|d_0| = 2 \sum_{j=k}^\infty \Gamma^j |d_0| + \sum_{j=k}^\infty \Gamma^j|\epsilon'_1|, \label{eq:eta1}
\end{equation}
or using the fact that from Eq.~\eqref{eq:relb0d0}, we have $d_0 \le (I - \gamma P_{\pi_*})^{-1}(-b_0+\epsilon'_1)$, as
\begin{equation}
|\eta_k| \leq \sum_{j=k}^\infty \Gamma^j |b_0| + \Gamma^k\sum_{j=0}^\infty (\gamma P_{\pi_*})^j\big(|b_0|+|\epsilon_1'|\big) = \sum_{j=k}^\infty \Gamma^j |b_0| + \Gamma^k\sum_{j=0}^\infty \Gamma^j\big(|b_0|+|\epsilon_1'|\big) = 2 \sum_{j=k}^{\infty} \Gamma^j |b_0|+ \sum_{j=k}^{\infty} \Gamma^j|\epsilon'_1|.  \label{eq:eta2}
\end{equation}
Now, using the definitions of $x_k$ and $y_k$ in Eqs.~\eqref{eq:defx} and~\eqref{eq:defy}, the bound on $|\eta_k|$ in Eq.~\eqref{eq:eta1} or~\eqref{eq:eta2}, and the fact that $\epsilon_0=0$, we obtain
\begin{align}
|l_k| &\leq \sum_{i=1}^{k} \Gamma^{i-1} \big[\Gamma^1|\epsilon_{k-i}|+|\epsilon'_{k-i+1}|\big] + \sum_{i=1}^{k-1} \sum_{j=i}^\infty \Gamma^j \big[(I+\Gamma^1)|\epsilon_{k-i}| + |\epsilon'_{k-i+1}|\big] + |\eta_k| \nonumber \\
&= \sum_{i=1}^{k-1} \Big(\Gamma^i+\sum_{j=i}^\infty(\Gamma^j+\Gamma^{j+1})\Big)|\epsilon_{k-i}| + \Gamma^k|\epsilon_0| + \sum_{i=1}^{k-1}\Big(\Gamma^{i-1}+\sum_{j=i}^\infty\Gamma^j\Big)|\epsilon'_{k-i+1}| + \Gamma^{k-1}|\epsilon'_1|+\sum_{j=k}^\infty \Gamma^j|\epsilon'_1| + h(k) \nonumber \\
&= 2 \sum_{i=1}^{k-1} \sum_{j=i}^\infty \Gamma^j |\epsilon_{k-i}| + \sum_{i=1}^{k-1} \sum_{j=i-1}^\infty \Gamma^j|\epsilon'_{k-i+1}| + \sum_{j=k-1}^\infty \Gamma^j|\epsilon'_1|+ h(k) = 2 \sum_{i=1}^{k-1} \sum_{j=i}^\infty \Gamma^j |\epsilon_{k-i}| + \sum_{i=0}^{k-1} \sum_{j=i}^\infty\Gamma^j|\epsilon'_{k-i}| + h(k), \label{eq:finall}
\end{align}
where we used the following definition
\begin{equation*}
h(k) \= 2 \sum_{j=k}^\infty \Gamma^j |d_0|, \quad\quad\quad \text{or} \quad\quad\quad h(k) \= 2\sum_{j=k}^\infty \Gamma^j |b_0|.
\end{equation*}

We end this proof by adapting the error propagation to CBMPI. 
As expressed by Eqs.~\ref{eq:def_error2_cbmpi} and~\ref{eq:def_error_cbmpi} in Sec.~\ref{analysis}, an analysis of CBMPI can be deduced from that we have just done by replacing $v_k$ with the auxiliary variable $w_k= (T_{\pi_{k}})^m v_{k-1} $ and $\epsilon_k$ with $(\gamma P_{\pi_k})^m \epsilon_{k-1}=\pp{m} \epsilon_{k-1} $. Therefore, using the fact that $\epsilon_0=0$, we can rewrite the bound of Eq.~\ref{eq:finall} for CBMPI as follows:
\begin{align}
l_k &\leq 2\sum_{i=1}^{k-1}\sum_{j=i}^\infty \Gamma^{j+m} |\epsilon_{k-i-1}| + \sum_{i=0}^{k-1}\sum_{j=i}^\infty \Gamma^j |\epsilon'_{k-i}| + h(k) \nonumber \\
&= 2\sum_{i=1}^{k-2}\sum_{j=m+i}^\infty \Gamma^{j} |\epsilon_{k-i-1}| + \sum_{i=0}^{k-1}\sum_{j=i}^\infty \Gamma^j |\epsilon'_{k-i}| + h(k) \label{eq:pointwiseCBMPI}.
\end{align}
%

\newpage
\section{Proof of Lemma~\ref{lemma:fromctolp}}
\label{proof:concenter}


For any integer $t$ and vector $z$, the definition of $\pp{t}$ and the H\"older's inequality imply that
\begin{equation}
\label{eq:holder}
\rho \pp{t} |z| = \left\|  \pp{t} |z| \right\|_{1,\rho} \le \gamma^t c_q(t) \| z \|_{q',\mu} =  \gamma^t c_q(t) \left(\mu |z|^{q'} \right)^{\frac{1}{q'}}.
\end{equation}
We define 
\begin{equation*}
K\=\sum_{l=1}^n \xi_l \left( \sum_{i \in {\cal I}_l} \sum_{j \in {\cal J}_{i}} \gamma^{j} \right), 
\end{equation*}
where $\{\xi_l\}_{l=1}^n$ is a set of non-negative numbers that we will specify later. We now have
\begin{align*}
\|f\|_{p,\rho}^p &= \rho |f|^p \\
&\le K^p \rho \left(\frac{\sum_{l=1}^n \sum_{i \in {\cal I}_l} \sum_{j \in {\cal J}_{i}}  \pp{j} |g_i| }{K} \right)^p
= K^p \rho \left(\frac{\sum_{l=1}^n \xi_l \sum_{i \in {\cal I}_l} \sum_{j \in {\cal J}_{i}}    \pp{j} \left(\frac{|g_i|}{\xi_l}\right) }{K} \right)^p \\
& \stackrel{(a)}{\le} K^p \rho \frac{\sum_{l=1}^n \xi_l \sum_{i \in {\cal I}_l} \sum_{j \in {\cal J}_{i}}     \pp{j} \left(\frac{|g_i|}{\xi_l}\right)^p }{K} 
= K^p \frac{\sum_{l=1}^n \xi_l \sum_{i \in {\cal I}_l} \sum_{j \in {\cal J}_{i}}    \rho \pp{j} \left(\frac{|g_i|}{\xi_l}\right)^p }{K} \\
&\stackrel{(b)}{\le}  K^p \frac{\sum_{l=1}^n \xi_l \sum_{i \in {\cal I}_l} \sum_{j \in {\cal J}_{i}}\gamma^{j}c_q(j)\left(\mu \left( \frac{|g_i|}{\xi_l}\right)^{pq'}\right)^{\frac{1}{q'}} }{K} \\
&= K^p \frac{\sum_{l=1}^n \xi_l \sum_{i \in {\cal I}_l} \sum_{j \in {\cal J}_{i}}\gamma^{j}c_q(j) \left( \frac{\|g_i\|_{pq',\mu}}{\xi_l}\right)^{p}}{K} \\
&\le K^p \frac{\sum_{l=1}^n \xi_l \left( \sum_{i \in {\cal I}_l}   \sum_{j \in {\cal J}_{i}}  \gamma^{j}c_q(j)\right) \left( \frac{\sup_{i \in {\cal I}_l}\|g_i\|_{pq',\mu}}{\xi_l}\right)^{p}}{K} \\
&\stackrel{(c)}{=} K^p \frac{\sum_{l=1}^n \xi_l \left(\sum_{i \in {\cal I}_l} \sum_{j \in {\cal J}_{i}}\gamma^{j}\right){\cal C}_{q}(l) \left( \frac{\sup_{i \in {\cal I}_l}\|g_i\|_{pq',\mu}}{\xi_l}\right)^{p}}{K},
\end{align*}
where {\bf (a)} results from Jensen's inequality, {\bf (b)} from Eq.~\ref{eq:holder}, and {\bf (c)} from the definition of ${\cal C}_{q}(l)$. Now, by setting $\xi_l = \big( {\cal C}_{q}(l)\big)^{1/p} \sup_{i \in {\cal I}_l} \|g_i\|_{pq',\mu}$, we obtain
\begin{equation*}
\|f\|_{p,\rho}^p \leq K^p \frac{\sum_{l=1}^n \xi_l \left(\sum_{i \in {\cal I}_l} \sum_{j \in {\cal J}_{i}}\gamma^{j}\right)}{K}=K^p,
\end{equation*}
where the last step follows from the definition of $K$.


\newpage
\section{Proof of Theorem~\ref{lpbound} \& other Bounds on the Loss}
\label{proof:lpbound}

\begin{proof}
We only detail the proof for AMPI (the proof being similar for CBMPI).
We define ${\cal I}=\{ 1, 2, \cdots, 2k\}$, the partition ${\cal I}=\{{\cal I}_1,{\cal I}_2,{\cal I}_3\}$ as ${\cal I}_1=\{1,\dots,k-1\}$, ${\cal I}_2=\{k,\dots,2k-1\}$, and ${\cal I}_3=\{2k\}$, and for each $i \in {\cal I}$
\begin{equation*}
g_i  = 
\left\{
\begin{array}{ll}
2\epsilon_{k-i} & \mbox{if }\quad 1 \le i \le k-1, \\
\epsilon'_{k-(i-k)} & \mbox{if }\quad k \le i \le 2k-1, \\
2 d_0 \mbox{~(or $2 b_0$)}& \mbox{if }\quad i=2k,
\end{array}
\right.
\mbox{~~~~~and~~~~~}
{\cal J}_i  = 
\left\{
\begin{array}{ll}
\{i, i+1, \cdots \} & \mbox{if }\quad 1 \le i \le k-1, \\
\{i-k, i-k+1, \cdots \} & \mbox{if }\quad k \le i \le 2k-1, \\
\{k, k+1, \cdots \} & \mbox{if }\quad i=2k.
\end{array}
\right.
\end{equation*}
Note that here we have divided the terms in the point-wise bound of Lemma~\ref{lemma:core2} into three groups: the {\em evaluation error} terms $\{\epsilon_j\}_{j=1}^{k-1}$, the {\em greedy step error} terms $\{\epsilon'_j\}_{j=1}^k$, and finally the residual term $h(k)$. With the above definitions and the fact that the loss $l_k$ is non-negative, Lemma~\ref{lemma:core2} may be rewritten as
\begin{equation*}
|l_k| \le \sum_{l=1}^3\sum_{i \in {\cal I}_l} \sum_{j \in {\cal J}_i} \Gamma^j|g_i|. 
\end{equation*}
The result follows by applying Lemma~\ref{lemma:fromctolp} and noticing that $\sum_{i=i_0}^{k-1} \sum_{j=i}^\infty \gamma^j = \frac{\gamma^{i_0}-\gamma^k}{(1-\gamma)^2}$.
\end{proof} 

Here in oder to show the flexibility of Lemma~\ref{lemma:fromctolp}, we group the terms differently and derive an alternative $L_p$-bound for the loss of AMPI and CBMPI. In analogy with the results of~\citet{Farahmand:2010}, this new bound shows that the last iterations have the highest influence on the loss (the influence exponentially decreases towards the initial iterations).
\begin{theorem}
\label{alternativebound}
With the notations of Theorem~\ref{lpbound}, after $k$ iterations, the loss of AMPI satisfies
\begin{equation*}
\|l_k \|_{p,\rho} \le 2 \sum_{i=1}^{k-1}\frac{\gamma^i}{1-\gamma}\left( {\cal C}_{q}^{i,i+1}\right)^{\frac{1}{p}} \|\epsilon_{k-i}\|_{pq',\mu}  + \sum_{i=0}^{k-1}\frac{\gamma^i}{1-\gamma}\left({\cal C}_{q}^{i,i+1}\right)^{\frac{1}{p}} \|\epsilon'_{k-i}\|_{pq',\mu} + g(k).
\end{equation*}
while the loss of CBMPI satisfies
\begin{equation*}
  \|l_k \|_{p,\rho} \le 2 \gamma^m \sum_{i=1}^{k-2}\frac{\gamma^i}{1-\gamma}\left( {\cal C}_{q}^{i,i+1}\right)^{\frac{1}{p}} \|\epsilon_{k-i-1}\|_{pq',\mu}  + \sum_{i=0}^{k-1}\frac{\gamma^i}{1-\gamma}\left({\cal C}_{q}^{i,i+1}\right)^{\frac{1}{p}} \|\epsilon'_{k-i}\|_{pq',\mu} + g(k).
\end{equation*}
\end{theorem}
\begin{proof}
Again, we only detail the proof for AMPI (the proof being similar for CBMPI).
We define $\cal I$, $(g_i)$ and $({\cal J}_i)$ as in the proof of Theorem~\ref{lpbound}. We then make as many groups as terms, i.e.,~for each $n \in \{1,2, \dots,  2k-1\}$, we define ${\cal I}_n=\{n\}$. The result follows by application of Lemma~\ref{lemma:fromctolp}.
\end{proof}


\newpage
\section{Proof of Lemma~\ref{l:actorCBMPI}}
\label{proof:actorCBMPI}

The proof of this lemma is similar to the proof of Theorem~1 in~\citet{Lazaric:2010}. Before stating the proof, we report the following two lemmas that are used in the proof. 
\begin{lemma}\label{l:vc-bound}
Let $\Pi$ be a policy space with finite VC-dimension $h=VC(\Pi)<\infty$ and $N$ be the number of states in the rollout set $\D_{k-1}$ drawn i.i.d.~from the state distribution $\mu$. Then we have 
\begin{equation*}
\mathbb P_{\D_{k-1}}\left[\sup_{\pi\in\Pi}\Big|\L_{k-1}^\Pi(\widehat{\mu};\pi)-\L_{k-1}^\Pi(\mu;\pi)\Big| > \epsilon\right] \leq \delta\;,
\end{equation*}
with $\epsilon = 16\Qmax\sqrt{\frac{2}{N}\big(h\log\frac{eN}{h} + \log\frac{8}{\delta}\big)}$.
\end{lemma}

\begin{proof} 
This is a restatement of Lemma~1 in~\citet{Lazaric:2010}. 
\end{proof}

\begin{lemma}\label{l:vc-bound2}
Let $\Pi$ be a policy space with finite VC-dimension $h=VC(\Pi)<\infty$ and $s^{(1)},\ldots,s^{(N)}$ be an arbitrary sequence of states. At each state we simulate $M$ independent rollouts of the form , then we have 
\begin{equation*}
\mathbb P\left[\sup_{\pi\in\Pi}\Big|\frac{1}{N}\sumSamples\frac{1}{M}\sum_{j=1}^MR_{k-1}^j\big(s^{(i,j)},\pi(s^{(i,j)})\big)-\frac{1}{N}\sumSamples Q_{k-1}\big(s^{(i,j)},\pi(s^{(i,j)})\big)\Big| > \epsilon\right] \leq \delta\;,
\end{equation*}
with $\epsilon = 8\Qmax\sqrt{\frac{2}{MN}\big(h\log\frac{eMN}{h} + \log\frac{8}{\delta}\big)}$.
\end{lemma}

\begin{proof} 
The proof is similar to the one for Lemma~\ref{l:vc-bound}. 
\end{proof}

\begin{proof}{\bf (Lemma~\ref{l:actorCBMPI}) } Let $a^*(\cdot)=\argmax_{a\in\action}Q_{k-1}(\cdot,a)$ be the greedy action. To simplify the notation, we remove the dependency of $a^*$ on states and use $a^*$ instead of $a^*(x_i)$ in the following. We prove the following series of inequalities:
\begin{align*}
\L^\Pi_{k-1}(\mu;\pi_k) &\stackrel{\text{(a)}}{\leq} \L^{\Pi}_{k-1}(\widehat{\mu};\pi_k)+\epsilon'_1 \quad\quad\quad\text{w.p. }1-\delta' \\
&= \frac{1}{N} \sumSamples\Big[Q_{k-1}(x_i,a^*)-Q_{k-1}\big(x_i,\pi_k(x_i)\big)\Big]+\epsilon'_1 \\
&\stackrel{\text{(b)}}{\leq} \frac{1}{N}\sumSamples\Big[Q_{k-1}(x_i,a^*)-\hQ_{k-1}\big(x_i,\pi_k(x_i)\big)\Big]+\epsilon'_1+\epsilon'_2\quad\quad\quad\text{w.p. }1-2\delta' \\
&\stackrel{\text{(c)}}{\leq} \frac{1}{N}\sumSamples\Big[Q_{k-1}(x_i,a^*)-\hQ_{k-1}\big(x_i,\pi^*(x_i)\big)\Big]+\epsilon'_1+\epsilon'_2 \\
&\leq \frac{1}{N}\sumSamples\Big[Q_{k-1}(x_i,a^*)-Q_{k-1}\big(x_i,\pi^*(x_i)\big)\Big]+\epsilon'_1+2\epsilon'_2\quad\quad\quad\text{w.p. } 1-3\delta' \\
&= \L^\Pi_{k-1}(\widehat{\mu};\pi^*)+ \epsilon'_1+2\epsilon'_2 \leq \L^\Pi_{k-1}(\mu;\pi^*)+2(\epsilon'_1+\epsilon'_2)\quad\quad\quad\text{w.p. } 1-4\delta' \\
&= \inf_{\pi\in\Pi}\L^\Pi_{k-1}(\mu;\pi)+2(\epsilon'_1+\epsilon'_2). 
\end{align*}

The statement of the theorem is obtained by $\probParam'=\probParam/4$. \\

\noindent
{\bf (a)} This follows from Lemma~\ref{l:vc-bound}. \\
\noindent
{\bf (b)} Here we introduce the estimated action-value function $\hQ_{k-1}$ by bounding 
\begin{equation*}
\sup_{\pi\in\Pi}\bigg[\frac{1}{N}\sumSamples\hQ_{k-1}\big(s^{(i)},\pi(s^{(i)})\big)-\frac{1}{N}\sumSamples Q_{k-1}\big(s^{(i)},\pi(s^{(i)})\big)\bigg]
\end{equation*}
using Lemma~\ref{l:vc-bound2}. \\
%
%
%
\noindent
{\bf (c)} From the definition of $\pi_k$ in CBMPI, we have 
\begin{equation*}
\pi_k=\argmin_{\pi\in\Pi} \widehat{\L}^{\Pi}_{k-1}(\widehat{\mu};\pi)=\argmax_{\pi\in\Pi}\frac{1}{N}\sumSamples\hQ_{k-1}\big(s^{(i)},\pi(s^{(i)})\big),
\end{equation*}
thus, $-1/N\sumSamples\hQ_{k-1}\big(s^{(i)},\pi_k(s^{(i)})\big)$ can be maximized by replacing $\pi_k$ with any other policy, particularly with
\begin{equation*}
\pi^* = \argmin_{\pi\in\Pi}\int_{\S}\left(\max_{a\in\A}Q_{k-1}(s,a) - Q_{k-1}\big(s,\pi(s)\big)\right) \mu(ds).
\end{equation*}
\end{proof}


\newpage
\section{Proof of Lemma~\ref{l:criticCBMPI}}
\label{proof:criticCBMPI}

\begin{figure}[t]
\centering
\includegraphics[scale=0.35]{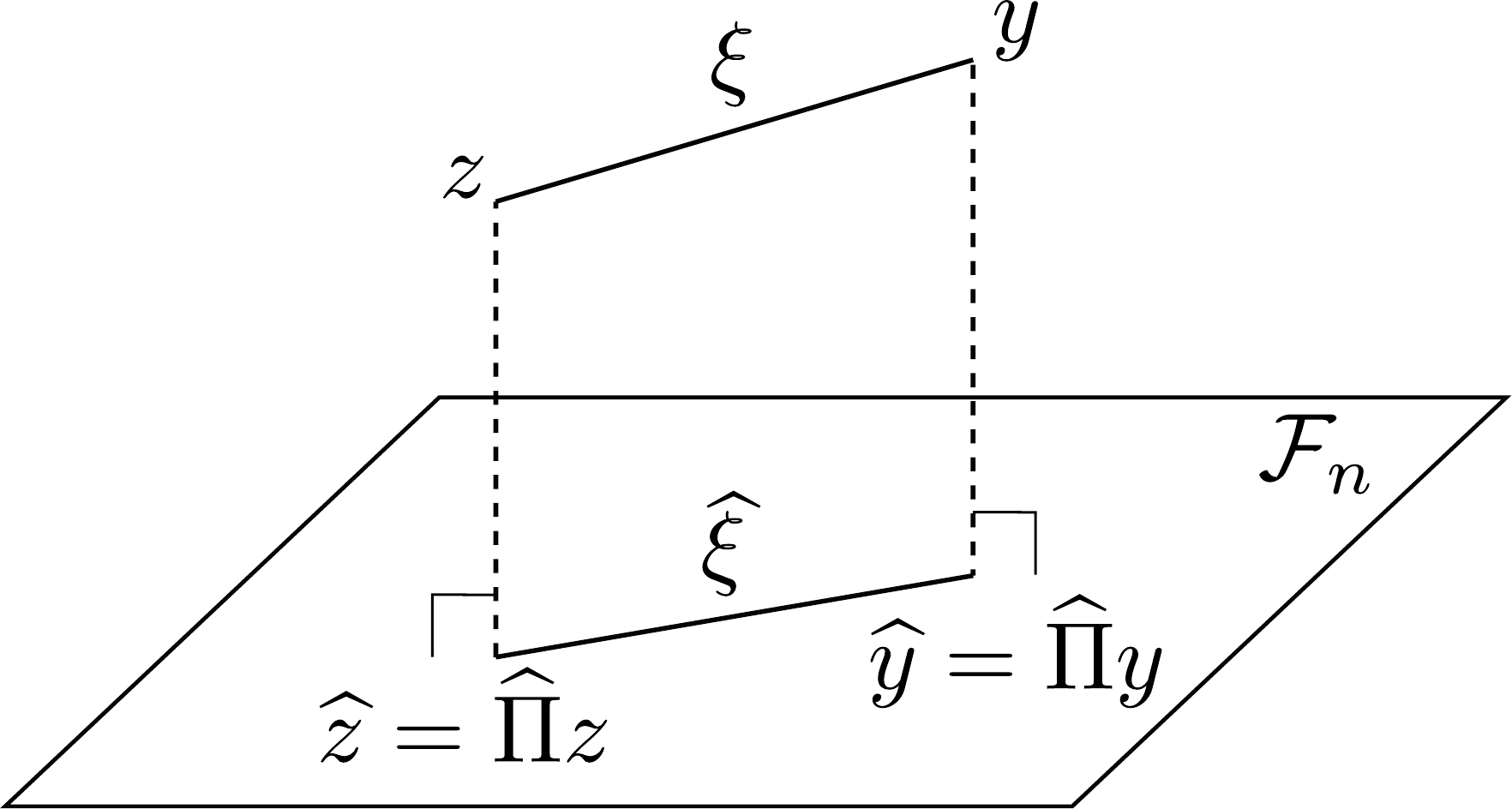}
\caption{The vectors used in the proof.}
\label{fig:critic}
\end{figure}

Let us define two $n$-dimensional vectors $z=\Big(\big[(T_{\pi_{k}})^mv_{k-1}\big](s^{(1)}),\ldots,\big[(T_{\pi_{k}})^mv_{k-1}\big](s^{(n)})\Big)^\top$ and $y=\big(\hv_{k}(s^{(1)}),\ldots,\hv_{k}(s^{(n)})\big)^\top$ and their orthogonal projections onto the vector space $\F_n$ as $\widehat{z}=\widehat{\Pi}z$ and $\widehat{y}=\widehat{\Pi}y=\big(\widetilde{v}_{k}(s^{(1)}),\ldots,\widetilde{v}_{k}(s^{(n)})\big)^\top$, where $\widetilde{v}_{k}$ is the result of linear regression and its truncation (by $\Vmax$) is $v_{k}$, i.e.,~$v_{k}=\mathbb{T}(\widetilde{v}_{k})$ (see Figure~\ref{fig:critic}). What we are interested is to find a bound on the regression error $\|z-\widehat{y}\|$ (the difference between the target function $z$ and the result of the regression $\widehat{y}$). We may decompose this error as
\begin{equation}
\label{eq:aa1}
\|z-\widehat{y}\|_n \leq \|\widehat{z}-\widehat{y}\|_n + \|z-\widehat{z}\|_n = \|\widehat{\xi}\|_n + \|z-\widehat{z}\|_n,
\end{equation}
where $\widehat{\xi}=\widehat{z}-\widehat{y}$ is the projected noise (estimation error) $\widehat{\xi}=\widehat{\Pi}\xi$, with the noise vector $\xi=z-y$ defined as $\xi_i=\big[(T_{\pi_{k}})^mv_{k-1}\big](s^{(i)})-\hv_{k}(s^{(i)})$. It is easy to see that noise is zero mean, i.e.,~$\mathbb{E}[\xi_i]=0$ and is bounded by $2\Vmax$, i.e.,~$|\xi_i|\leq 2\Vmax$. We may write the estimation error as
\begin{equation*}
\|\widehat{z}-\widehat{y}\|_n^2=\|\widehat{\xi}\|_n^2=\langle\widehat{\xi},\widehat{\xi}\rangle=\langle\xi,\widehat{\xi}\rangle,
\end{equation*}
where the last equality follows from the fact that $\widehat{\xi}$ is the orthogonal projection of $\xi$. Since $\widehat{\xi}\in\F_n$, let $f_\alpha\in\F$ be any function whose values at $\{s^{(i)}\}_{i=1}^n$ equals to $\{\xi_i\}_{i=1}^n$. By application of a variation of Pollard's inequality~(Gy\"{o}rfi et al.,~2002), we obtain
\begin{equation*}
\langle\xi,\widehat{\xi}\rangle = \frac{1}{n}\sum_{i=1}^n\xi_if_\alpha(s^{(i)}) \leq 4\Vmax\|\widehat{\xi}\|_n\sqrt{\frac{2}{n}\log\left(\frac{3(9e^2n)^{d+1}}{\delta'}\right)},
\end{equation*}
with probability at least $1-\delta'$. Thus, we have
\begin{equation}
\label{eq:aa2}
\|\widehat{z}-\widehat{y}\|_n = \|\widehat{\xi}\|_n \leq 4\Vmax\sqrt{\frac{2}{n}\log\left(\frac{3(9e^2n)^{d+1}}{\delta'}\right)}.
\end{equation}
From Eqs.~\ref{eq:aa1} and~\ref{eq:aa2}, we have
\begin{equation}
\label{eq:aa3}
\|(T_{\pi_{k}})^mv_{k-1} - \widetilde{v}_{k}\|_{\widehat{\mu}} \leq \|(T_{\pi_{k}})^mv_{k-1} - \widehat{\Pi}(T_{\pi_{k}})^mv_{k-1}\|_{\widehat{\mu}} + 4\Vmax\sqrt{\frac{2}{n}\log\left(\frac{3(9e^2n)^{d+1}}{\delta'}\right)},
\end{equation}
where $\widehat{\mu}$ is the empirical norm induced from the $n$ i.i.d.~samples from $\mu$. \\

Now in order to obtain a random design bound, we first define $f_{\widehat{\alpha}_*}\in\F$ as $f_{\widehat{\alpha}_*}(s^{(i)})=\big[\widehat{\Pi}(T_{\pi_{k}})^mv_{k-1}\big](s^{(i)})$, and then define $f_{\alpha_*}=\Pi(T_{\pi_{k}})^mv_{k-1}$ that is the best approximation (w.r.t.~$\mu$) of the target function $(T_{\pi_{k}})^mv_{k-1}$ in $\F$. Since $f_{\widehat{\alpha}_*}$ is the minimizer of the empirical loss, any function in $\F$ different than $f_{\widehat{\alpha}_*}$ has a bigger empirical loss, thus we have
\begin{align}
\label{eq:aa4}
\|f_{\widehat{\alpha}_*} - (T_{\pi_{k}})^mv_{k-1}\|_{\widehat{\mu}} &\leq \|f_{\alpha_*} - (T_{\pi_{k}})^mv_{k-1}\|_{\widehat{\mu}} \leq 2\|f_{\alpha_*} - (T_{\pi_{k}})^mv_{k-1}\|_\mu \nonumber \\
&+ 12\Big(\Vmax + \|\alpha_*\|_2\;\sup_x\|\phi(x)\|_2\Big)\sqrt{\frac{2}{n}\log\frac{3}{\delta'}}\;,
\end{align}
with probability at least $1-\delta'$, where the second inequality is the application of a variation of Theorem~11.2 in the book by 
Gy\"{o}rfi et al.,~(2002) with $\|f_{\alpha_*} - (T_{\pi_{k}})^mv_{k-1}\|_\infty\leq\Vmax + \|\alpha^*\|_2\;\sup_x\|\phi(x)\|_2$. Similarly, we can write the left-hand-side of Equation~\ref{eq:aa3} as
\begin{equation}
\label{eq:aa5}
2\|(T_{\pi_{k}})^mv_{k-1}-\widetilde{v}_{k}\|_{\widehat{\mu}} \geq 2\|(T_{\pi_{k}})^mv_{k-1}-\mathbb{T}(\widetilde{v}_{k})\|_{\widehat{\mu}} \geq \|(T_{\pi_{k}})^mv_{k-1}-\mathbb{T}(\widetilde{v}_{k})\|_\mu - 24\Vmax\sqrt{\frac{2}{n}\Lambda(n,d,\delta')},
\end{equation}
with probability at least $1-\delta'$, where $\Lambda(n,d,\delta')=2(d+1)\log n + \log\frac{e}{\delta'} + \log\big(9(12e)^{2(d+1)}\big)$. Putting together Equations~\ref{eq:aa3},~\ref{eq:aa4}, and~\ref{eq:aa5} and using the fact that $\mathbb{T}(\widetilde{v}_{k})=v_{k}$, we obtain
\begin{align*}
\|\eta_{k}\|_{2,\mu} = \|(T_{\pi_{k}})^mv_{k-1} - v_{k}\|_\mu &\leq 2\Bigg(2\|(T_{\pi_{k}})^mv_{k-1} - f_{\alpha_*}\|_\mu + 12\Big(\Vmax + \|\alpha_*\|_2\;\sup_x\|\phi(x)\|_2\Big)\sqrt{\frac{2}{n}\log\frac{3}{\delta'}} \\
&+ 4\Vmax\sqrt{\frac{2}{n}\log\left(\frac{3(9e^2n)^{d+1}}{\delta'}\right)}\Bigg) + 24\Vmax\sqrt{\frac{2}{n}\Lambda(n,d,\delta')}.
\end{align*}
The result follows by setting $\delta=3\delta'$ and some simplification.

\newpage
\section{Experimental Results}
\label{s:experiments}
			
In this section, we report the empirical evaluation of CBMPI and compare it to DPI and LSPI. In the experiments, we show that CBMPI, by combining policy and value function approximation, can improve over DPI and LSPI. In these experiments, we are using the same setting as in~\citet{Gabillon11CP} to facilitate the comparison. 




\subsection{Setting}\label{ss:setting}


We consider the mountain car (MC) problem with its standard formulation in which the action noise is bounded in $[-1,1]$ and $\gamma = 0.99$. The value function is approximated using a linear space spanned by a set of radial basis functions (RBFs) evenly distributed over the state space. 

Each CBMPI-based algorithm is run with the same fixed budget $B$ per iteration. CBMPI splits the budget into a rollout budget $B_R = B(1-p)$ used to build the training set of the greedy step and a critic budget $B_C = Bp$ used to build the training set of the evaluation step , where $p \in (0,1)$ is the critic ratio. The rollout budget is divided into $M$ rollouts of length $m$ for each action in $\A$ and each state in the rollout set $\Data'$, i.e.,~$B_R = m  M  N |\A|$. The critic budget is divided into one rollout of length $m$ for each action in $\A$ and each state in the rollout set $\Data$, i.e.,~$B_C = m  n |\A|$. 

 
In Fig.~\ref{fig:ExpMC}, we report the performance of DPI, CBMPI, and LSPI. In MC, the performance is evaluated as the number of steps-to-go with a maximum of $300$. 
 The results are averaged over $1000$ runs. We report the performance of DPI and LSPI at $p=0$ and $p=1$, respectively. DPI can be seen as a special case of CBMPI where $p=0$.
We tested the performance of DPI and CBMPI on a wide range of parameters $(m,M,N,n)$ but we only report their performance for the best choice of $M$ ($M=1$ was the best choice in all the experiments) and different values of $m$.
 

\subsection{Experiments}\label{ss:empR}


\begin{figure*}[t]
\begin{minipage}[t]{0.5\linewidth}
\center
\includegraphics[width=0.9\textwidth]{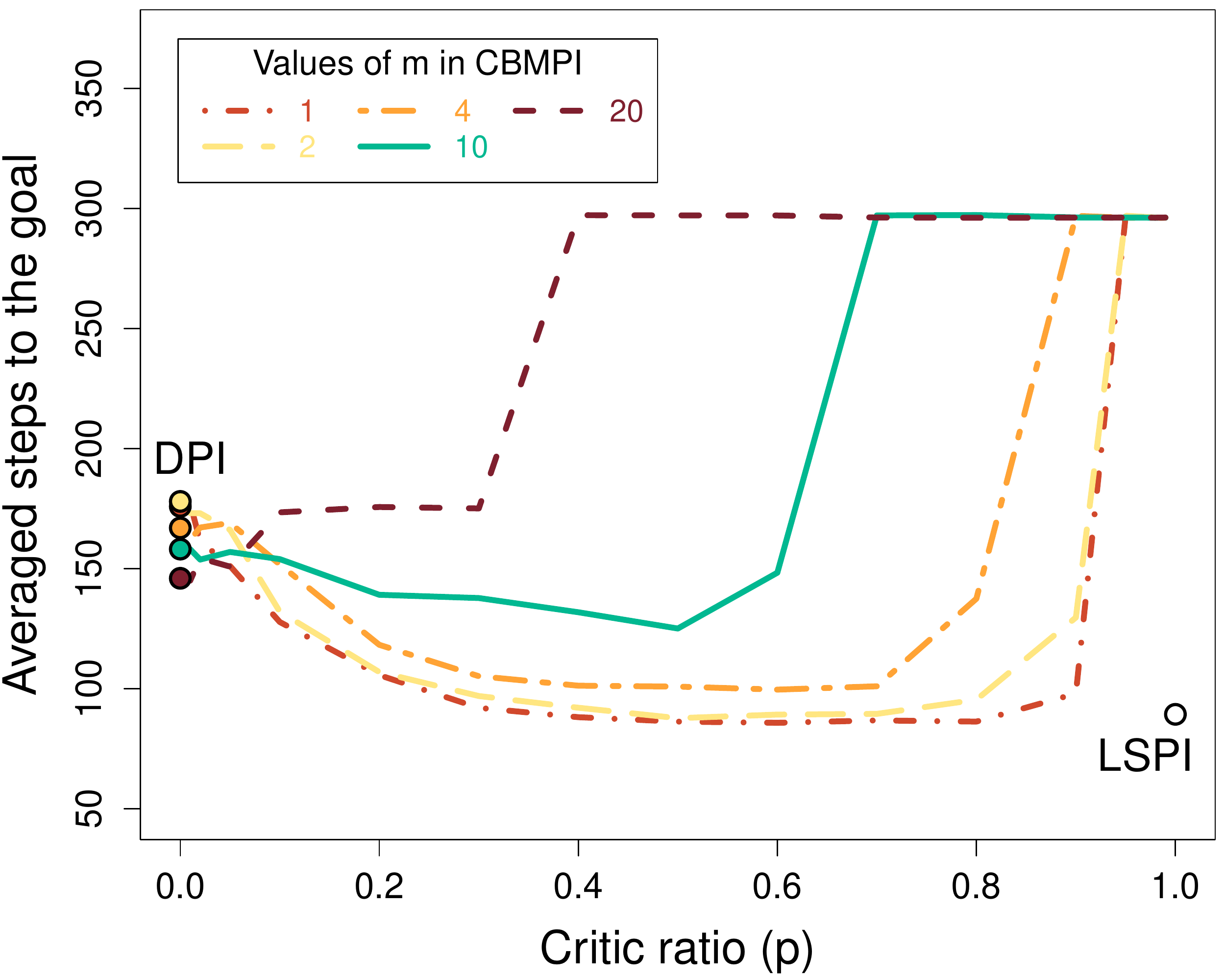} 
\end{minipage} 
\begin{minipage}[t]{.5\linewidth}
\center
\includegraphics[width=0.9\textwidth]{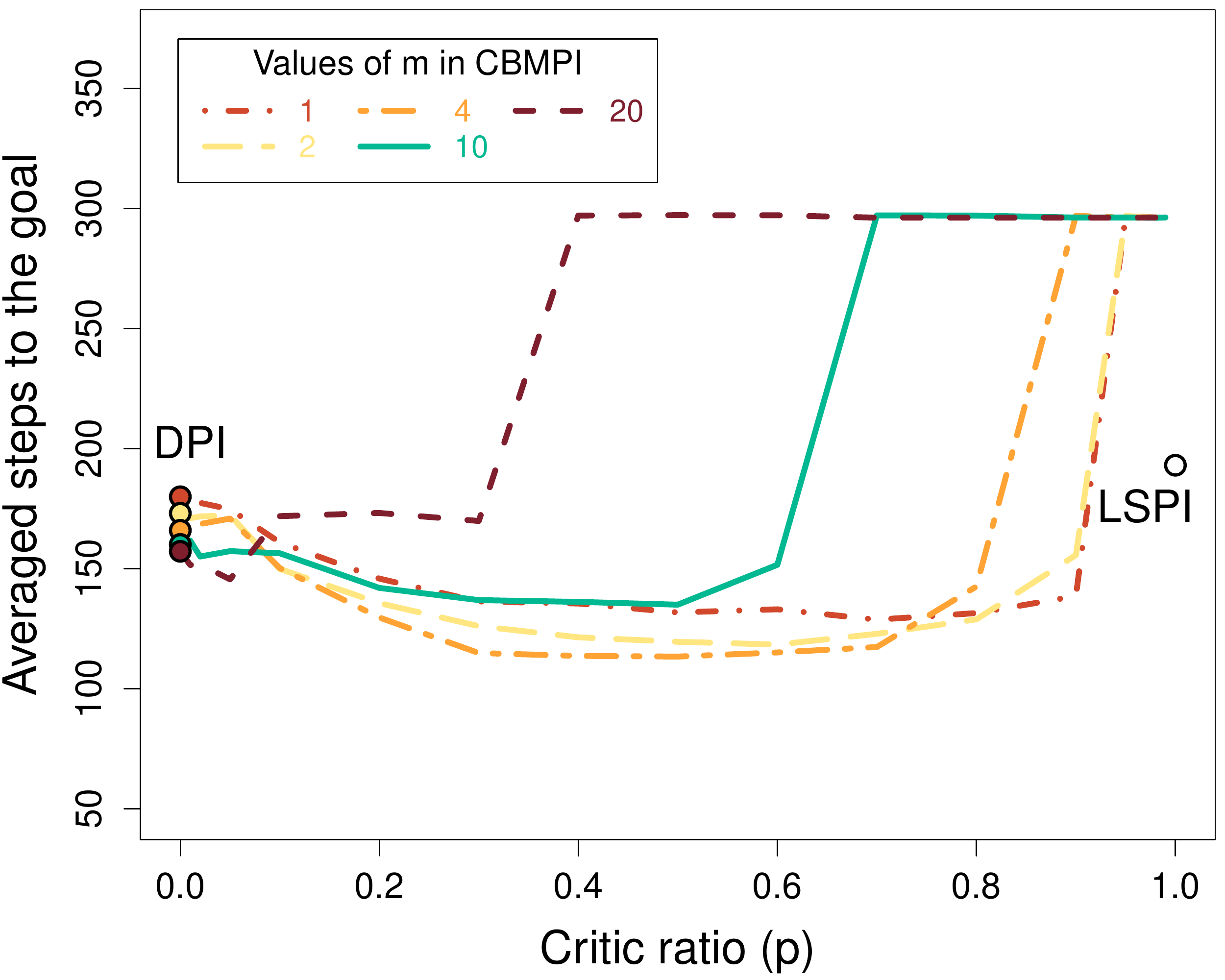} 
\end{minipage}
\caption{Performance of the learned policies in mountain car with two different $2\times2$ RBF grids, the one with good approximation of the value function is on the left and the one with poor performance in approximating the value function is on the right. The total budget $B$ is set to 200. The objective is to minimize the number of steps to the goal.}
\label{fig:ExpMC}
\end{figure*}

As discussed in Remark~\ref{r:m}, the parameter $m$ balances between the error in evaluating the value function and the error in evaluating the policy. The value function approximation error tends to zero for large values of $m$. Although this would suggest to have large values for $m$, the size of the rollout sets would correspondingly decrease as $N = O(B/m)$ and $n = O(B/m)$, thus decreasing the accuracy of both the regression and classification problems. This leads to a trade-off between long rollouts and the number of states in the rollout sets. The solution to this trade-off strictly depends on the capacity of the value function space $\F$. A rich value function space would lead to solve the trade-off for small values of $m$. On the other hand, when the value function space is poor, or as in the DPI case, $m$ should be selected in a way to guarantee a sufficient number of informative rollouts, and at the same time, a large enough rollout sets.

Figure~\ref{fig:ExpMC} shows the learning results in MC with budget $B=200$. On the left panel, the function space is rich enough to approximate $v^*$. Therefore LSPI has almost optimal results (about $80$ steps to reach the goal). On the other hand, DPI achieves a poor performance of about $150$ steps, which is obtained by setting $m=12$ and $N=5$. We also report the performance of CBMPI for different values of $m$ and $p$.  When $p$ is large enough, the value function approximation becomes accurate enough so that the best solution is to have $m=1$. This both corresponds to rollouts built almost entirely on the basis of the approximated value function and to a large number of states in the training set $N$. For $m=1$ and $p\approx 0.8$, CBMPI achieves a slightly better performance than LSPI. 

In the next experiment, we show that CBMPI is able to outperform both DPI and LSPI when $\F$ has a lower accuracy. The results are reported on the right panel of Figure~\ref{fig:ExpMC}. The performance of LSPI now worsens to $190$ steps. Simultaneously one can notice $m=1$ is no longer the best choice for CBMPI. Indeed in the case where $m=1$, CBMPI becomes an approximated version of the value iteration algorithm relying on a function space not rich enough to approximate $v*$. Notice that relying on this space is still better than setting the value function to zero which is the case in DPI. Therefore, we notice an improvement of CBMPI over DPI for $m=4$ which trade-off between the estimates of the value function and the rewards collected by the rollouts. Combining those two, CBMPI also improves upon LSPI.

\end{document}

%% file: symbolsdef.tex
\newdimen\pHeight
\newdimen\pLower
\newdimen\pLineWidth
\newdimen\pKern
\newdimen\pIR
\newsavebox{\Cbox}
\newsavebox{\vertCmplx}
\newdimen\Cheight
\newdimen\Cwidth
\sbox{\Cbox}{\rm C}
\Cheight=\ht\Cbox
\Cwidth=\wd\Cbox
\advance\Cheight by \pHeight
\sbox{\vertCmplx}{\rule[\pLower]{\pLineWidth}{\Cheight}}
\sbox{\Cbox}{\usebox{\Cbox}\kern\pKern\usebox{\vertCmplx}}
\wd\Cbox=\Cwidth
\def\Re{\mathbb{R}}
\def\Nat{{\rm I\kern\pIR N}}

\def\log{\mathop{\rm log}}

\newcommand{\bsmall}[1]{\vspace{#1}\begin{small}}
\newcommand{\esmall}[1]{\end{small}\vspace{#1}}

\newcommand{\argmax}{\operatorname*{argmax}} 
\newcommand{\argmin}{\operatorname*{argmin}}

\newcommand{\greedy}{\operatorname*{{\cal G}}}

\newcommand\pp[1]{\Gamma^{#1}}

\def\={\stackrel{\Delta}{=}}
\def\E{\mathbb{E}}
\def\R{\mathbb{R}}
\def\G{{\cal G}}
\def\F{{\cal F}}

\def\cf{{see} }

\def\A{{\mathcal{A}}}

\def\D{{\mathcal{D}}}
\def\F{{\mathcal{F}}}
\def\G{{\mathcal{G}}}

\def\L{{\mathcal{L}}}

\def\S{{\mathcal{S}}}

\def\X{{\mathcal{X}}}

\def\vec0{{\boldsymbol{0}}}

\def\setP{{\mathbb{P}}}

\renewcommand{\Re}{\mathbb R}

\newcommand{\discount}{\gamma}
\newcommand{\state}{S}
\newcommand{\action}{\mathcal{A}}

\newcommand{\pol}{\pi}
\newcommand{\polSpace}{\Pi}

\newcommand{\Data}{\mathcal{D}}

\newcommand{\Rmax}{R_{\max}}

\newcommand{\Vmax}{V_{\max}}
\newcommand{\hV}{\widehat{v}}
\newcommand{\hv}{\widehat{v}}

\newcommand{\Qfun}{Q}

\newcommand{\Qmax}{Q_{\max}}
\newcommand{\hQ}{\widehat{Q}}

\newcommand{\nSamples}{N}
\newcommand{\sumSamples}{\sum_{i=1}^{\nSamples}}

\newcommand{\probParam}{\delta}

\newcommand{\q}{Q}